%% file: main.tex
\documentclass[11pt]{article}

\usepackage{tgpagella}
\usepackage{mathpazo}
\usepackage{inconsolata}
\usepackage[letterpaper, margin=0.75in, top=0.6in, bottom=0.7in]{geometry}
\usepackage{natbib}
\setcitestyle{authoryear,round,citesep={;},aysep={,},yysep={;}}
\input{math_commands.tex}

\usepackage[utf8]{inputenc} %
\usepackage[T1]{fontenc}    %
\usepackage{hyperref}       %
\usepackage{url}            %
\usepackage{booktabs}       %
\usepackage{amsmath}        %
\usepackage{amsthm}         %
\usepackage{amsfonts}       %
\usepackage{nicefrac}       %
\usepackage{microtype}      %
\usepackage{xcolor}         %
\usepackage{graphicx}       %
\usepackage{enumitem}
\usepackage{tikz}
\usetikzlibrary{arrows.meta,positioning}
\usepackage{multirow}       %
\usepackage{caption}  %
\definecolor{darkblue}{rgb}{0, 0, 0.5}
\hypersetup{colorlinks=true, citecolor=darkblue, linkcolor=darkblue, urlcolor=darkblue}
\usepackage[most]{tcolorbox}

\definecolor{takeawaycol}{HTML}{2AA198}
\newtcolorbox{takeaway}[1]{%
  enhanced, breakable,
  colback=takeawaycol!8,
  colframe=takeawaycol,
  boxrule=0.5pt,
  arc=2.5mm,
  left=3mm, right=3mm, top=2mm, bottom=2mm,
  before upper={\textbf{#1.}\ }%
}

\newtheorem{proposition}{Proposition}

\begin{document}

\thispagestyle{plain}
\vspace*{-0.4in}
\begingroup
\renewcommand{\thefootnote}{\fnsymbol{footnote}}
\begin{center}
{\LARGE\bfseries Scaling Discovery through Test-Time Communication}\\[10pt]
{\normalsize Jongho Park$^{b}$\footnotemark[1] \quad Vasilis Kontonis$^{m}$ \quad Shivam Garg$^{m}$}\\[1pt]
{\normalsize Akshay Krishnamurthy$^{m}$ \quad Dimitris Papailiopoulos$^{m}$}\\[2pt]
{\normalsize\bfseries $^{b}$UC Berkeley \quad $^{m}$Microsoft Research}\\[8pt]
\end{center}
\footnotetext[1]{This work was done during an internship at Microsoft Research.}
\renewcommand{\thefootnote}{}
\footnotetext[0]{Emails: jjhpark@berkeley.edu, \{vkontonis, shigarg, akshay.krishnamurthy, dimitriosp\}@microsoft.com}
\endgroup
\vspace{-0.5cm}
\noindent\rule{\textwidth}{0.4pt}
\vspace{-0.2cm}
\renewcommand{\abstractname}{}

\begin{abstract}
\vspace{-0.5cm}
Science advances not in isolation but through collaboration, 
yet existing agentic systems capture little of this.
Whether communicating agents help remains an open question
with mixed prior results.
We show that \emph{test-time communication can substantially outperform
independent parallel attempts on challenging tasks, where sharing a
breakthrough can push the whole group forward}. 
We first study the effect of scaling multi-agent test-time communication, 
where agents have no predefined roles and communicate via a shared directory,
on ARC-AGI-3, a benchmark requiring novel problem solving.
We find that a team of $k$ communicating agents, team@$k$, matches the success rate of 
$4k$ independent agents, and this advantage grows with 
$k$, suggesting gains compound with scale.
The effect is not merely efficiency: a task that no single agent can solve, a
team of agents can solve reliably. Furthermore, these gains
transfer to research-oriented tasks, given sufficient compute. On
polyomino packing, communicating agents outperform best@$k$
and exceed the prior best-known score. On MNIST classifier compression, communication
surpasses the best-known human solution. A team of four agents produced a
1,957-byte classifier submission achieving 99.4\% test accuracy,
smaller than both the best-known human solution and the best single-agent
result. These gains are not unconditional. 
Independent agents may outperform communication when 
compute is limited or when a clear measure of progress is absent. 
However, under sufficient compute and clear feedback, multi-agent communication consistently yields stronger results.

\end{abstract}
\vspace{0.2cm}

\begin{figure}[h!]
    \centering
    \makebox[\linewidth][c]{%
    \includegraphics[width=\linewidth]{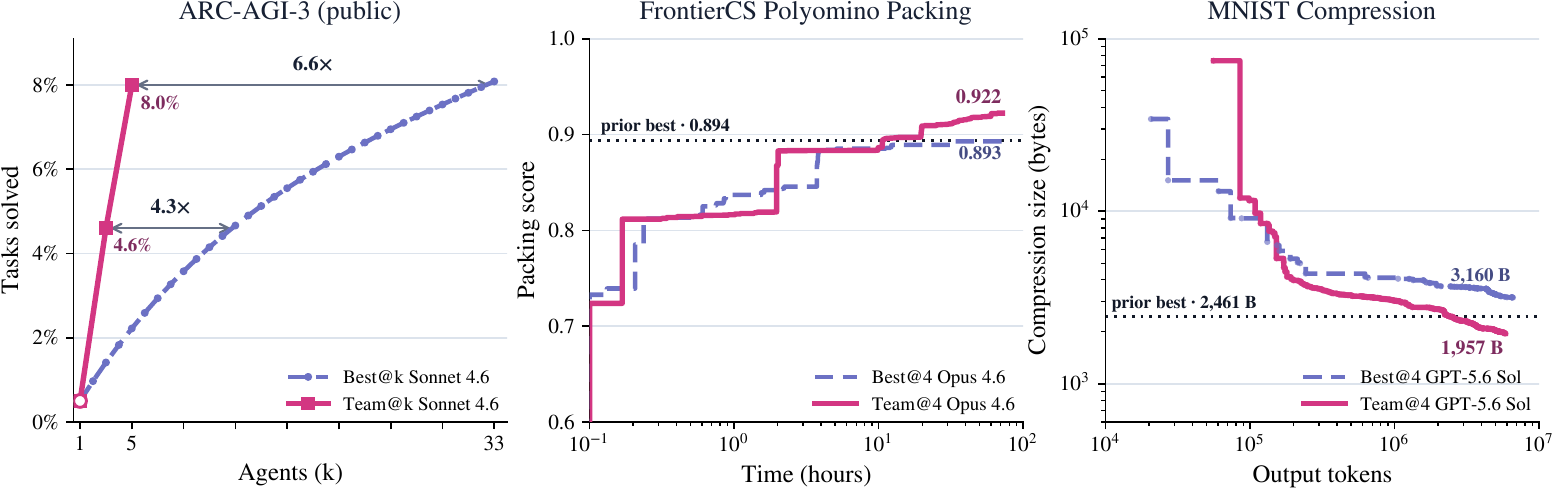}%
    }
    \caption{\textbf{Communication outperforms independent agents in open-ended
    tasks.} With communication, a team of $k$ communicating agents not only outperforms best@$k$ but scales success rate 
    on ARC-AGI-3 faster. On algorithmic and ML optimization tasks, team@$k$ beats single-agent runs 
    and best-known human solutions, yielding new state-of-the-art results.}
    \label{fig:page1_figure}
\end{figure}

\input{sections/intro}
\input{sections/method}
\input{sections/results}
\input{sections/related}
\input{sections/conclusion}

\subsubsection*{Acknowledgments}
We thank Microsoft Research AI Frontiers for supporting this work 
and Ziyang Cai for insightful discussions during the project.
This work was done during the first author's internship at Microsoft Research.

\bibliography{ref}
\bibliographystyle{plainnat}

\newpage
\appendix
\input{sections/setup_appendix}
\input{sections/proof_appendix}

\end{document}

%% file: math_commands.tex
\usepackage{amsmath,amsfonts,bm}

\def\eqref#1{equation~\ref{#1}}

\def\1{\bm{1}}

\DeclareMathAlphabet{\mathsfit}{\encodingdefault}{\sfdefault}{m}{sl}
\SetMathAlphabet{\mathsfit}{bold}{\encodingdefault}{\sfdefault}{bx}{n}

%% file: sections/intro.tex
\section{Introduction}

For test-time compute, it is standard to sample $k$ independent outputs in parallel from a large language model (LLM) and aggregate them, \emph{e.g.}, by taking the majority answer or
the best@$k$ when a verifier is available~\citep{wang2023selfconsistency,
brown2024largelanguagemonkeys,li2024moreagents,huang2025self,snell2025scalingtesttime}.
Independent attempts, however, leave discoveries made during one run unavailable to
guide another while its search is still unfolding.
This opportunity is especially crucial for LLM agents, which operate a terminal, execute
code, use tools, and observe results over horizons of hours or
days~\citep{yang2024sweagent,chan2025mlebench,merrill2026terminalbench}.
Their trajectories contain intermediate results, failed experiments, and reusable artifacts
that could help other agents avoid dead ends or build on promising approaches. \looseness=-1

Communication between agents would be the natural remedy, 
but the evidence for its benefits is mixed.
While keeping the best of $k$ independent candidates improves with $k$,
having agents exchange opinions may not~\citep{choi2025debateorvote}.
Multi-agent debate often fails to outperform chain-of-thought
prompting despite using substantially more inference compute
\citep{smit2024goingmad,zhang2025stopovervaluing}.
Agents can defer to an incorrect majority or dilute expertise by compromising 
between expert and non-expert judgments~\citep{wu2025canagentsdebate,pappu2026experts}.
Most recently, \citet{anthropic2026multiagentsystems} report that coordinating agents
find more vulnerabilities than independent agents, but use more tokens and search more broadly.
Within the common search scope, tokens per vulnerability are comparable, 
leaving the efficiency gains from communication unclear.
On the other hand, multi-agent systems clearly help when a task can be decomposed into subtasks and solved in
parallel~\citep{kim2025scienceofscaling}, but such gains are limited to cleanly decomposable tasks.
Can communication offer more than just parallelism?

Science suggests that it can. Scientists coordinate by adjusting their research to the results
achieved by others~\citep{polanyi1962republic}, while accumulated empirical evidence helps 
distinguish competing explanations and turn varied efforts into collective progress~\citep{strevens2020knowledge}.
This process is iterative rather than merely parallel.
Collaboration should therefore do more than divide and conquer. One agent's
observation redirects another's search, and partial discoveries combine into a result
that no individual reaches alone. 
We ask whether agent communication can produce the same effect.

To this end, we study a \emph{minimal form of test-time communication} 
in which identical agents receive the same open-ended objective and communicate 
through a shared workspace, without predefined roles or a central orchestrator. 
Autonomously and asynchronously, the agents can
exchange intermediate results, failures, and artifacts while their search is still
unfolding. Our central comparison pits a team of $k$ communicating agents, \textbf{team@$k$}, against
the best result from $k$ independent agents, \textbf{best@$k$}, under the same per-agent resources.
This isolates the value of communication from the gains produced by additional parallel
attempts.

We first study \texttt{ARC-AGI-3}~\citep{arcagi3report2026} using Claude Sonnet~4.6. 
In terms of full-game success rate,
team@3 matches best@13 and team@5 matches best@33, making them as effective as $4.3$ and $6.6$
times as many independent agents (Figure~\ref{fig:page1_figure}). More strikingly,
communicating teams solve games that independent agents rarely or never solve.
These benefits transfer to longer-horizon research problems. 
On Frontier-CS polyomino packing~\citep{mang2026frontiercs}, where agents pack polyominoes into a minimum-area rectangle, team@4 scores 0.922 with Claude
Opus~4.6 and 0.910 with Sonnet~4.6, compared with 0.893 and 0.891, respectively,
for the strongest solo run of each model and 0.894 for the prior best-known result.
On MNIST classifier compression under an accuracy constraint, we evaluate team@4 with GPT-5.6~Sol over 96 hours. 
The best team produces a state-of-the-art 1,957-byte classifier at 99.4\% accuracy, beating the 2,461-byte best-known human solution and single-agent baselines. 

These results suggest that test-time communication is most useful when agents have
enough compute to build on one another's discoveries and can objectively assess
whether those discoveries improve on earlier results.
Sufficient compute allows teams to overcome initial coordination costs, while an
accessible verifier, such as level success or a solution scorer, helps agents
decide which discoveries to adopt.
We call this mechanism \emph{verified progress sharing}.
By building on successive breakthroughs from different members, a communicating team can 
turn parallel exploration into cumulative progress.

When agents cannot verify intermediate progress, however, communication
may offer less benefit.
On Terminal-Bench 2.0, where available feedback may not reliably rank intermediate
solutions, we find that team@2 improves over a single attempt but does not outperform
independent pass@2.
This distinction also offers an explanation for the prior negative results, 
where agents lack a verifier to distinguish better solutions from
worse ones and therefore abandon stronger solutions in favor of weaker
ones~\citep{choi2025debateorvote,pappu2026experts}.
Together, our findings show that communication can offer more than divide and conquer parallelism
and help identify the conditions under which test-time communication excels.

\vspace{30pt}
Our contributions are as follows:
\begin{enumerate}[
      leftmargin=*,
      labelindent=0pt,
      topsep=2pt,
      itemsep=1pt,
      parsep=4pt
  ]
    \item \textbf{Communicating agents show stronger scaling than independent agents.}
    On the \texttt{ARC-AGI-3} suite of games,
    matching the solve rate of team@5 requires 33 independent agents.
    The multiplier grows with team size,
    from $4.3\times$ for team@3 to $6.6\times$ for team@5, 
    suggesting that the gains from communication compound with scale.
    In fact, the solve rate of team@$k$ increases more rapidly with $k$ than that of best@$k$,
    as seen in Figure~\ref{fig:page1_figure}.
    \item \textbf{Test-time communication unlocks tasks that single agents cannot.}
    Team@3 lifts the solve rate on the game \texttt{FT09} on \texttt{ARC-AGI-3} 
    from 9.4\% for a single agent to 90\%. 
    More strikingly, team@5 solves the game \texttt{LP85},
    which remains unsolved across 64 single-agent trials, 65\% of the time.
    Team@5 also improves the average furthest level reached over best@5 
    even on games that remain unsolved.
    \item \textbf{Communication achieves state-of-the-art results on algorithmic and ML optimization tasks.}
    Communication sustains progress over multi-day horizons.
    On Frontier-CS polyomino packing, a team of agents outperforms independent agents 
    and surpasses the prior best-known score.
    On ML compression, it produces a MNIST classifier substantially smaller than 
    those found by independent agents or the best-known human solution.
    Agents reach these solutions by refining one another's discoveries 
    and combining complementary improvements, including ideas from peers' failed approaches.
    \item \textbf{Communication benefits from verifiable progress and sufficient compute.}
    At low budgets, team@$k$ incurs a coordination tax that outweighs its benefits.
    We identify \emph{verified progress sharing} as a mechanism through which agents build on
    successive discoveries to overcome the coordination tax given enough compute.
    On Terminal-Bench 2.0, where available feedback may not reliably rank intermediate solutions,
    team@2 improves over a single attempt but does not outperform pass@2.
    This helps reconcile prior results by identifying compute and verification 
    as conditions that shape communication benefits.
\end{enumerate}

%% file: sections/method.tex
\section{Agentic Communication at Test-time}
\label{sec:method}

\begin{table}[t!]
    \centering
    \small
    \setlength{\tabcolsep}{4.5pt}
    \caption{\textbf{Summary table.} We evaluate multi-agent communication and compare the best of $k$ independent agents with a team of $k$ communicating agents on three tasks.
    All agents are run using the GitHub Copilot CLI and communication harness we describe in Section~\ref{sec:method}.
    }
    \label{tab:setup-summary}
    \begin{tabular}{@{}l|llll@{}}
        \toprule
        Benchmark & ARC-AGI-3 & Polyomino packing & MNIST Compression \\
        \midrule
        Model & Sonnet 4.6 & Sonnet, Opus 4.6 & GPT-5.6 Sol \\
        Budget per agent & per-level action budget & 3-72 hours & 96 hours \\
        Team size $k$ & 3, 5 & 3, 4 & 4 \\
        Metric & game solve rate & packing score & qualified bytes \\
        Prior best-known result & -- & 0.894 & 2,461 B  \\
        Independent best@$k$ & 1.4\% ($k$=3), 2.2\% ($k$=5)  & 0.893  & 3,160 B \\
        Communicating team@$k$ & \textbf{4.6\%} ($k$=3), \textbf{8.0\%} ($k$=5)  & \textbf{0.945} & \textbf{1,957 B} \\
        \bottomrule
    \end{tabular}
\end{table}

We study a minimal form of test-time communication instantiated through a
shared workspace among identical agents. In each
team@$k$ trial, the harness launches $k$ CLI agents concurrently
in the same task container. They use the same model, tools, task
instruction, and communication prompt. Each agent has a separate model
context and designated scratch directory, while all agents share the task
filesystem and communication artifacts.

Agents communicate directly through an append-only \textit{communication log},
which acts as an asynchronous broadcast channel. Additional shared records
contain adopted approaches, disconfirming evidence, and a
score log of the approaches so far. Without assigned roles or
a central orchestrator, the only protocol-level allocation, enforced by a
synchronization primitive, is slot ownership.
To ensure that agents claim distinct approaches without collision,
agents race to create numbered slot directories using an atomic filesystem
operation. 

The communication prompt (Appendix~\ref{app:comm-prompt}) 
specifies how agents use these mechanisms and
discourages premature convergence. After claiming a slot, each agent is asked
to declare a distinct approach by considering the already-claimed slots.
During execution, agents publish concise findings with timestamps and
reproducible evidence whenever they make notable progress, allowing peers to
reproduce or build upon their results. Results placed on the leaderboard
should include measured outcomes, reproduction instructions, approach
lineage, or known counterevidence. An agent is to adopt a peer's approach
only after observing a clearly better result, and even after adoption, it
should preserve one meaningful variation.

The protocol combines diverse exploration with evidence-based adoption to
discourage convergence on poor solutions.
Atomic slot ownership, together with the requirement to pursue distinct
approaches, helps preserve diversity.
Verifier scores give agents a basis for filtering out poorly performing
approaches when deciding what to adopt.
This contrasts with the task settings of~\citet{pappu2026experts}, where
agents lack verifier scores for agent solutions and can dilute expertise
by compromising between expert and non-expert judgments.

\subsection{Experimental Setup}
\label{sec:setup}

We study multi-agent communication on the following three tasks, 
each covering a different aspect of open-ended research tasks.

\paragraph{ARC-AGI-3}~\citep{arcagi3report2026} tests novel problem solving and interactive
discovery by asking agents to solve unfamiliar grid-world games without instructions, inferring
the rules, the goal, and the effect of each control from play alone.
Each game comprises $l_{\max} \in [6,10]$ levels, and progress depends on carrying forward
what was learned in earlier ones.
This leveled structure makes breakthroughs cleanly measurable.
Real-world research problems may not offer such a clean signal. A genuine conceptual advance
in a domain such as approximation factors for NP-hard problems 
may move the reported number by only a small constant, 
an improvement easily lost in run-to-run variance.

We evaluate on all 25 public games with Claude Sonnet 4.6 under the benchmark’s native
per-level action budget and measure solve rate as the proportion of trials that clear all
levels successfully.
Though more recent models such as GPT-6 Astra now succeed easily on the benchmark,
we focus on one model, especially one that does not saturate, to isolate the effect of test-time communication.
For teams, an agent that exhausts its action budget is terminated, while the remaining agents continue. 
We run 64 single-agent trials per game and label a game unsolved by single agents if no trial clears all levels.
For teams of three and five agents, we run 20 trials for every combination of game and team size.

\paragraph{Frontier-CS polyomino packing}~\citep{mang2026frontiercs} tests algorithmic
optimization on an NP-hard problem. Agents write and repeatedly improve a C++17 program that packs reflected and
rotated polyominoes into a minimum-area rectangle.
The scorer evaluates the submitted solution on 70 hidden test cases and returns continuous
partial credit in $[0,1]$.
The \textit{packing score} is the mean reward across the 70 cases, and the highest valid
submission in a run is retained.
The best published score is 0.894, obtained by \citet{qu2026coral} using four Claude
Opus 4.6 agents.
We test on both Claude Sonnet~4.6 and Opus~4.6 with teams of three or four agents.

\paragraph{MNIST Classifier Compression} tests empirical ML research.
Agents must train and compress a self-contained MNIST classifier~\citep{lecun1998gradient}.
An artifact qualifies only at 99.4\% test accuracy or better and 
agents cannot inspect the test images, labels, or individual errors, as described in
Appendix~\ref{app:mnist-setup}. Among qualifying
artifacts, only compressed size counts. We measure size using a deterministic gzip-9 compression of
the submission, which includes inference code and model weights (they may be separate files or one file,
depending on the agent's design).
To the best of our knowledge, the best prior result comes from the open-source model introduced by
\citet{gandhi2024tinymnist}, which we reproduce as a 2,461-byte submission under our
artifact format.
We test a team of four GPT-5.6~Sol agents over a 96-hour period.
Further implementation details can be found in Appendix~\ref{app:setup}.

\paragraph{Runtime Environment.}
Every agent operates through GitHub Copilot CLI~\citep{github2026copilotcli},
with each trial isolated as a Harbor task~\citep{harborframework2026}.
Single-agent trials contain one CLI agent. Team@$k$ trials contain $k$ agents in the
same task container. Agents receive the same task, tools, and base filesystem,
but have private scratch directories. Within a team,
agents communicate explicitly. They append timestamped claims, evidence, failures,
and adoption events to shared logs, publish measured candidates to a board, and
use file locks to serialize changes to a shared graded artifact. Distinct trials
share neither files nor messages.
Hidden evaluation data are kept outside the agent container. Polyomino submissions are
sent to a separate scorer service, while MNIST submissions are evaluated by a host-owned
oracle in a fresh network-disabled container. These services return evaluation feedback
but never expose the underlying cases, labels, or retained program outputs.

\paragraph{Metrics.}
Our central comparison is the outcome of $k$ communicating agents (\textbf{team@$k$}) versus
the best outcome among $k$ independent agents (\textbf{best@$k$}).
We measure both metrics against wall-clock time and against total output tokens to account for compute.
For \texttt{ARC-AGI-3}, a level counts as solved if any of the $k$ agents clears it, 
and a game counts as solved if any of the $k$ agents clears all levels.
For a finite ARC solo pool with $s$ successes among $n$ trials, we compute best@$k$
exactly, without replacement, as
$\mathrm{best@}k = 1 - \binom{n-s}{k}/\binom{n}{k}$.
We apply the same calculation level by level and average over games,
giving each game equal weight. 
Throughout the paper, solve rate refers to the final solve rate averaged over all 25 games, and per-level rates are labeled explicitly.
We measure team@$k$ and best@$k$ as the best-so-far packing score (mean packed-cell density across 70 cases) or MNIST classifier compression size (gzipped submission size including code and weights).

Within every comparison, communicating and independent agents use the same model, maximum reasoning
effort, longest context-length setting, and per-agent resource allocation.
Appendix~\ref{app:common-setup} provides the communication protocol, prompts, and Copilot
CLI versions.

%% file: sections/results.tex
\section{Experimental Results}

\subsection{ARC-AGI-3}
\label{sec:arc_agi_results}

\begin{table}[t!]
  \centering
  \caption{\textbf{ARC-AGI-3 per-level solve rate.}
  Proportion of trials reaching within $d$ levels of a full solve, averaged over
  25 games. Test-time communication pays off most where the task is hardest. At $k=5$,
  team@$k$ leads best@$k$ by $1.2\times$ at shallow depths and by $3.6\times$ at
  a full solve. Communication also scales with team size. Matching team@$k$'s final solve rate
  takes a pool of 13 independent agents at
  $k=3$ and 33 agents at $k=5$.
}
  \label{tab:arc_solve_level}
  \begin{tabular}{@{}rcccccc@{}}
    \toprule
    & \multicolumn{3}{c}{$k=3$} & \multicolumn{3}{c@{}}{$k=5$} \\
    \cmidrule(lr){2-4}\cmidrule(l){5-7}
    Levels from target & team@3 & best@3 & best@13 & team@5 & best@5 & best@33 \\
    \midrule
    5         & 52.8 & 45.4 & 63.9 & 64.1 & 52.8 & 69.5 \\
    4         & 28.2 & 25.8 & 46.5 & 33.6 & 32.9 & 55.4 \\
    3         & 16.6 &  9.7 & 20.8 & 22.9 & 12.8 & 31.3 \\
    2         &  7.2 &  5.0 &  9.3 & 16.2 &  6.1 & 15.8 \\
    1         &  4.8 &  2.6 &  6.3 &  10.0 &  3.7 & 10.2 \\
    (Solved) 0 & \textbf{4.6} & 1.4 & 4.7 & \textbf{8.0} & 2.2 & 8.1 \\
    \bottomrule
  \end{tabular}
\end{table}

Despite \texttt{ARC-AGI-3}'s visual simplicity, the most capable LLM-based agents fail on most games without a specialized harness. 
For this benchmark, we use Claude Sonnet 4.6 agents through Copilot, which reach a final success rate below 1\%.

Table~\ref{tab:arc_solve_level} reports the percentage of runs that reach success or come within $d$ levels of it. 
Across 25 games, team@$k$ outperforms best@$k$ at every level for both $k=3$ and $k=5$, and the advantage widens with depth, from $1.2\times$ to $3.6\times$ at $k=5$. 
The leftmost panel of Figure~\ref{fig:page1_figure} shows a further encouraging pattern in which team@$k$ rises monotonically with team size and faster than best@$k$.
Matching a communicating team takes 4.3--6.6$\times$ as many independent agents, a multiplier that itself grows with $k$, 
though returns may diminish beyond five agents.

\begin{figure}[t!]
\centering
\begin{minipage}[t]{.72\linewidth}
\vspace{0pt}
\centering
{\footnotesize \textbf{(a)} Communication effect in ARC-AGI-3\par}
\vspace{2pt}
\includegraphics[width=\linewidth,trim=0 0 0 20pt,clip]{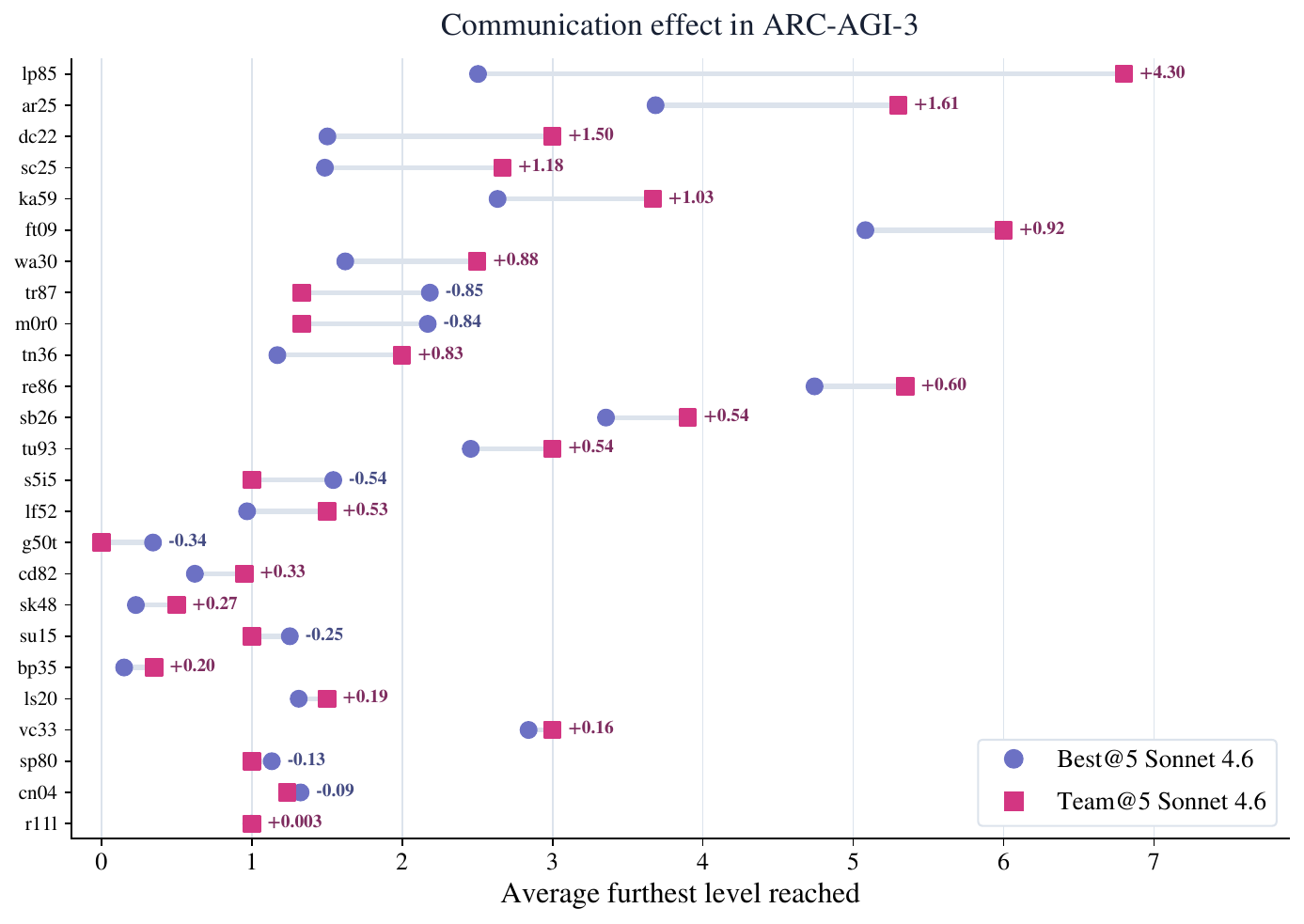}
\end{minipage}\hfill
\begin{minipage}[t]{.25\linewidth}
\vspace{0pt}
\centering
{\footnotesize \textbf{(b)} Final solve rate\par}
\vspace{2pt}
\scriptsize
\setlength{\tabcolsep}{2.2pt}
\renewcommand{\arraystretch}{0.83}
\begin{tabular}{@{}lrr@{}}
\toprule
Game & best@5 & team@5 \\
\midrule
\texttt{LP85} & 0.0 & \textbf{65.0} \\
\texttt{AR25} & 7.8 & \textbf{20.0} \\
\texttt{DC22} & 0.0 & 0.0 \\
\texttt{SC25} & 0.0 & 0.0 \\
\texttt{KA59} & 0.0 & 0.0 \\
\texttt{FT09} & 39.9 & \textbf{100.0} \\
\texttt{WA30} & 0.0 & 0.0 \\
\texttt{TR87} & 0.0 & 0.0 \\
\texttt{M0R0} & 0.0 & 0.0 \\
\texttt{TN36} & 0.0 & 0.0 \\
\texttt{RE86} & 0.0 & 0.0 \\
\texttt{SB26} & 7.8 & \textbf{15.0} \\
\texttt{TU93} & 0.0 & 0.0 \\
\texttt{S5I5} & 0.0 & 0.0 \\
\texttt{LF52} & 0.0 & 0.0 \\
\texttt{G50T} & 0.0 & 0.0 \\
\texttt{CD82} & 0.0 & 0.0 \\
\texttt{SK48} & 0.0 & 0.0 \\
\texttt{SU15} & 0.0 & 0.0 \\
\texttt{BP35} & 0.0 & 0.0 \\
\texttt{LS20} & 0.0 & 0.0 \\
\texttt{VC33} & 0.0 & 0.0 \\
\texttt{SP80} & 0.0 & 0.0 \\
\texttt{CN04} & 0.0 & 0.0 \\
\texttt{R11L} & 0.0 & 0.0 \\
\bottomrule
\end{tabular}
\end{minipage}
\caption{\textbf{Net communication effect on each ARC-AGI-3 game.}
\textbf{(a)} Average furthest level reached by team@5 and best@5, sorted by absolute difference.
Communication helps on 18 of the 25 games and hurts on 7, and the best gain is $+4.30$ levels 
on \texttt{LP85} while no loss exceeds $0.85$.
\textbf{(b)} Success rate
across 64 single-agent trials and 20 team@5 trials.
The aggregate gain comes from a handful of games, as most remain out of
reach for Sonnet 4.6.
Nevertheless, where communication helps, it helps substantially. Across the four games with
any team@5 success, solve rate rises from $13.8\%$ to $50.0\%$ on average.
Most strikingly, \texttt{LP85} goes from no single-agent success to $65\%$.
}
\label{fig:arc-dumbbell}
\vspace{-6pt}
\end{figure}

Figure~\ref{fig:arc-dumbbell} breaks these aggregate results down by game. Across the four games
solved at least once by team@5, the average solve rate rises from $13.8\%$ to $50.0\%$. The
improvement is especially striking on \texttt{LP85}, where 64 single-agent trials yield no
successes, yet team@5 achieves a $65\%$ solve rate. Nor is the effect limited to the largest
team. On \texttt{FT09}, the solve rate climbs from $25.9\%$ with best@3 to $90\%$ with team@3.
Although these gains are concentrated, this reflects the limits of the base model rather than of
communication, since most of the 25 games remain beyond Sonnet 4.6 under either setting.

Moreover, zero-percent solve rates obscure meaningful gains in progress. Figure~\ref{fig:arc-dumbbell}(a) reports the average furthest level reached. Under team@5, several never-solved
games still advance by roughly a full level, showing that test-time communication produces real
progress across most games. The largest gain is $+4.30$ levels, while no loss exceeds $0.85$
levels. When coordination fails, the cost is negligible. When it succeeds, the gains can be
substantial.

\begin{takeaway}{Agents in collaboration outperform agents in isolation}
Communication increases the number of games solved in \texttt{ARC-AGI-3} by $3.3\text{--}3.6\times$,
while team@$k$'s per-level solve rate is always higher than best@$k$'s (Table~\ref{tab:arc_solve_level}).
Matching team@$k$'s final solve rate requires approximately $4.3\times$ as many
independent agents for $k=3$ and $6.6\times$ as many for $k=5$.
Figure~\ref{fig:page1_figure} suggests that the benefit grows with team size.
\end{takeaway}

\begin{figure}[t!]
\centering
\includegraphics[width=.49\linewidth]{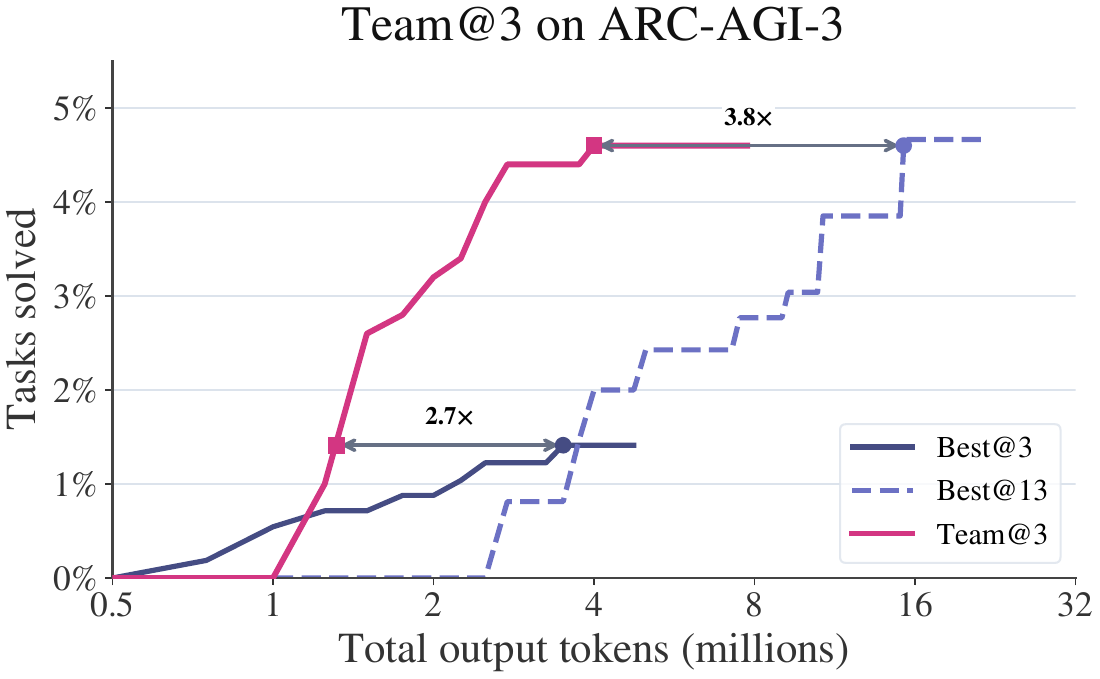}\hfill
\includegraphics[width=.49\linewidth]{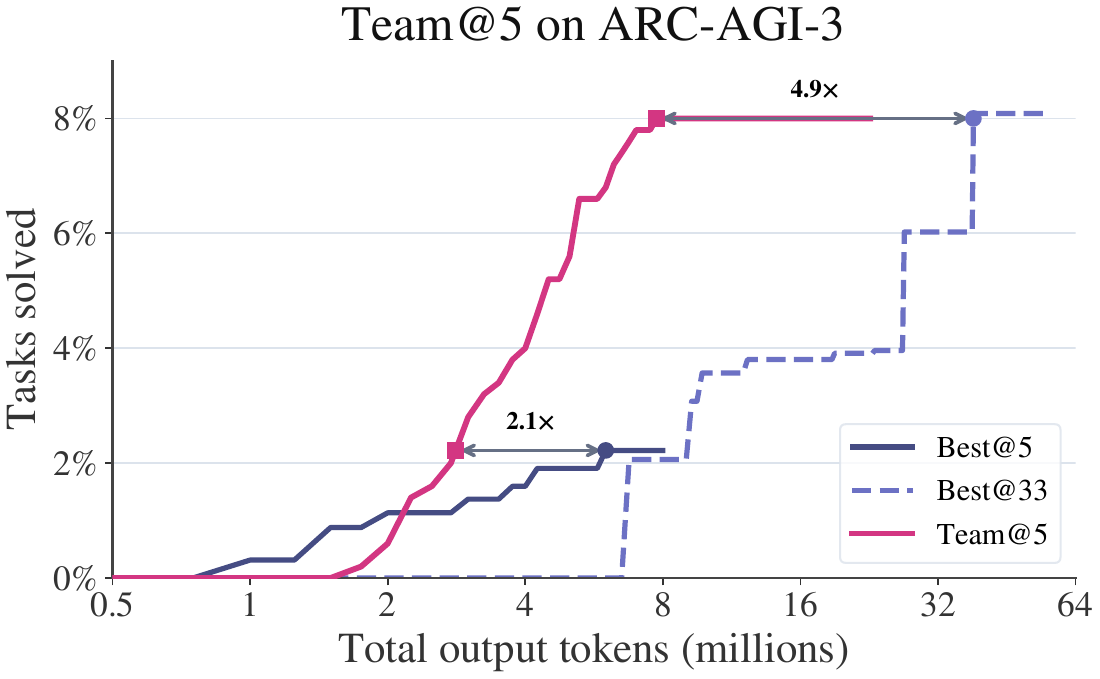}
\caption{\textbf{Token efficiency of communicating agents.} Solve rate against
cumulative output tokens for team@$k$, best@$k$, and the number of independent agents needed
to match the team's final solve rate. Best@$k$ never
reaches the team's final rate at any budget. For $k$=3 and $k$=5, respectively, matching that rate
takes $3.8\times$ and $4.9\times$ more tokens than the team spends.
The gain is not immediate. Best@$k$ leads in a low-compute regime ($\le$ 400K output tokens per agent),
a coordination tax that the team overcomes at larger budgets.}
\label{fig:arc-token-efficiency}
\vspace{-5pt}
\end{figure}

\begin{figure}[t!]
\centering
\includegraphics[width=\linewidth]{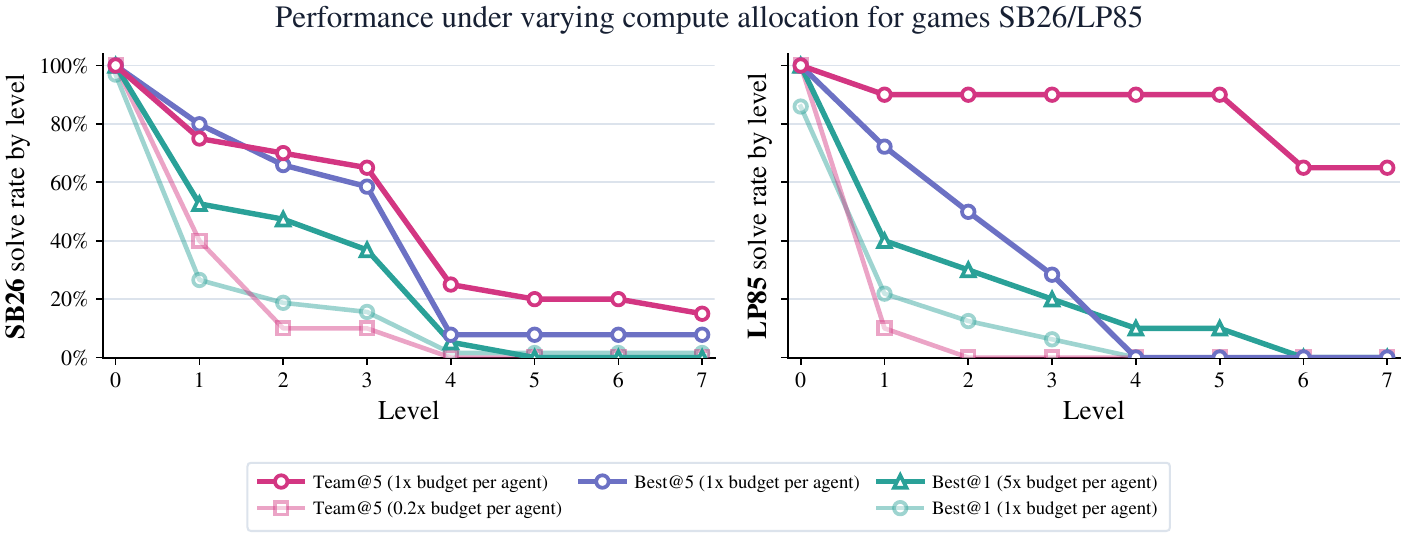}
\caption{\textbf{Effective collaboration depends on compute allocation.} For games SB26 and LP85,
we vary the per-agent action budget while holding the total action budget
at $5\times$ the native budget (solid line). We also consider the case in which a team must divide one agent's native budget (faded line).
The marker shape indicates per-agent budget, while the color indicates agent configuration.
When each agent is given the native budget, per-level team@5 leads at every level past the first,
ahead of both best@5 and a single agent given $5\times$ the budget. 
Cutting the same five agents to $0.2\times$ each drops team@5 
below a single agent.}
\label{fig:arc-budget-constraint}
\vspace{-5pt}
\end{figure}

Is test-time communication simply a matter of spending more tokens than independent agents do?
Figure~\ref{fig:arc-token-efficiency} paints a more nuanced picture.
By the end of evaluation, team@$k$ spends roughly 8 million output tokens with three agents and 24 million with five, nearly twice what best@$k$ spends.
However, to solve more than 1\% of \texttt{ARC-AGI-3} tasks, 
team@$k$ reaches any given accuracy with fewer total tokens, with best@3 spending $2.7\times$ as many as team@3 and best@5 spending $2.1\times$ as many as team@5.
Matching the solve rate of test-time communication is even costlier, as pools of 13 and 33 independent agents
must spend as much as 3.8--4.9$\times$ the team's tokens.

We further study \texttt{SB26} and \texttt{LP85}, two games that benefit from communication, with
two action-budget ablations. We first give a single agent $5\times$ the native action budget,
matching the total budget of standard team@5 or best@5. This configuration tests whether a longer 
horizon of actions of a single agent can match the performance of a team of communicating or independent
agents. We then reduce team@5's per-agent action budget to $0.2\times$, so
its total action budget matches that of a single agent.

At matched total budget, communication wins. On \texttt{SB26}, best@5 leads early
at $80\%$ to $75\%$, but team@5 overtakes at level 2 and solves the final level in
$15\%$ of trials against roughly $8\%$ for best@5 and zero for the long-horizon
single agent. \texttt{LP85} separates further, where team@5 holds $90\%$ through
level 5 and finishes at $65\%$ while best@5 falls from $72\%$ to zero by level 4
and the single agent peaks at $10\%$. A longer horizon alone does not substitute
for coordination.
This advantage, however, vanishes once the budget is reduced.
At $0.2\times$ per agent, team@5 drops to zero after level 1 on \texttt{LP85} and
after level 3 on \texttt{SB26}, in both cases losing to a single agent spending
the same total. Communication pays only when each agent has enough budget to
explore on its own.

\begin{takeaway}{Communication is more time- and token-efficient given sufficient compute}
Communicating teams pay an initial coordination tax but eventually
overtake best@$k$ if enough compute is available. 
Matching team@$k$ with independent agents requires
$3.8\times$ as many output tokens for $k=3$ and $4.9\times$ 
as many for $k=5$ (Figure~\ref{fig:arc-token-efficiency}).
As seen in Figure~\ref{fig:arc-budget-constraint}, teams can also 
outperform single agents that have the budget of an entire team.
However, with a smaller budget, teams may divide already constrained
resources and perform worse than independent agents.
\end{takeaway}

\paragraph{Relative human action efficiency.}
Throughout this paper, we take solve rate as our primary metric, since our
question is whether communication and larger teams, that is, more test-time
compute, lead to more tasks solved in \texttt{ARC-AGI-3}.
The official metric, by contrast, scores a run not only by what it clears but 
by how economically it plays, using a
metric called relative human action efficiency (RHAE)~\citep{arcagi3report2026}.
We define RHAE below and then use it to show
that test-time communication improves per-agent efficiency as well.

\begin{figure}[t!]
\centering
\includegraphics[width=.9\linewidth]{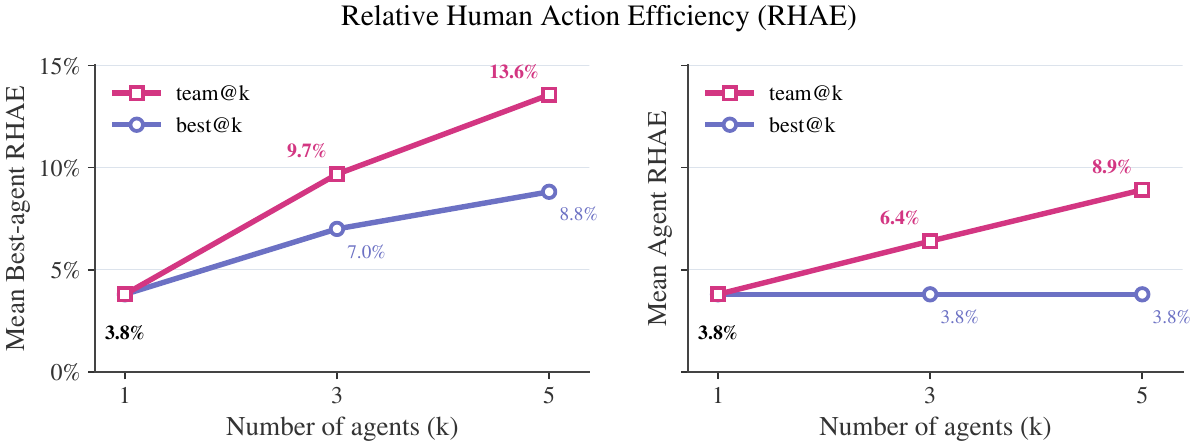}
\caption{\textbf{Relative human action efficiency of communicating agents.}
RHAE (Eq.~\ref{eq:rhae-level}) for the team's best agent and average agent across all games.
RHAE scores trajectories based on whether and how quickly an agent
solves each level relative to human performance. The figure shows that
communication makes each agent better off. On average, each agent's action efficiency
increases with team size, while the best agent in the team is more efficient
than the best independent agent.}
\label{fig:arc-agent-rhae}
\end{figure}

An action is a discrete interaction that changes the game
state, so reasoning, tool calls, and read-only inspection are not charged. For level $\ell$
of game $e$, let $a_{\ell,e}$ be the actions the agent spends and $h_{\ell,e}$ the
human baseline, taken as the upper-median best action count over first-time human players.
The level score is
\begin{equation}
    S_{\ell,e} = \min\left\{1.15,\;
        \left(\frac{h_{\ell,e}}{a_{\ell,e}}\right)^{\!2}\right\},
    \label{eq:rhae-level}
\end{equation}
with $S_{\ell,e}=0$ for any level the agent never completes. Game scores weight
level $\ell$ by $w_\ell=\ell$ and are capped by the weighted fraction of levels completed,
\begin{equation}
    E_e = \min\left\{
        \frac{\sum_{\ell=1}^{d} w_\ell}{\sum_{\ell=1}^{n} w_\ell},\;
        \frac{\sum_{\ell=1}^{n} w_\ell\, S_{\ell,e}}{\sum_{\ell=1}^{n} w_\ell}
    \right\},
    \label{eq:rhae-env}
\end{equation}
where $d$ of $n$ levels are completed, and the reported RHAE is the mean of $E_e$ over
games. 

Using RHAE, Figure~\ref{fig:arc-agent-rhae} compares independent Sonnet 4.6 agents with
communicating teams across all 25 games. 
Communication improves not only the strongest team
member but also the average member. The best agent's mean RHAE
rises from $3.8\%$ for a single agent to $9.7\%$ with team@3 and $13.6\%$ with team@5, exceeding
the corresponding best@3 and best@5 values of $7.0\%$ and $8.8\%$. This gap shows that the gain is
not merely the result of selecting the best outcome from more agents.
The average communicating agent also improves with team size, 
while the average independent agent remains at $3.8\%$ (by definition). Strikingly, the
average agent in team@5 matches the best of five independent agents, at $8.9\%$ versus $8.8\%$.
Together, the two panels show that communication improves both the average and the strongest team
member, rather than improving team performance only by pooling more attempts.

\begin{takeaway}{Communication boosts the action efficiency of each agent in the team}
Communication improves both the strongest agent and the average agent. At
$k=5$, the best agent in a communicating team reaches $13.6\%$ mean RHAE,
compared with $8.8\%$ for best@$5$. The average team agent reaches $8.9\%$,
performing as efficiently as the best of five independent agents.
\end{takeaway}

\subsection{Polyomino Packing}
\label{sec:polyomino}

\begin{figure}[h!]
\centering
\includegraphics[width=\linewidth]{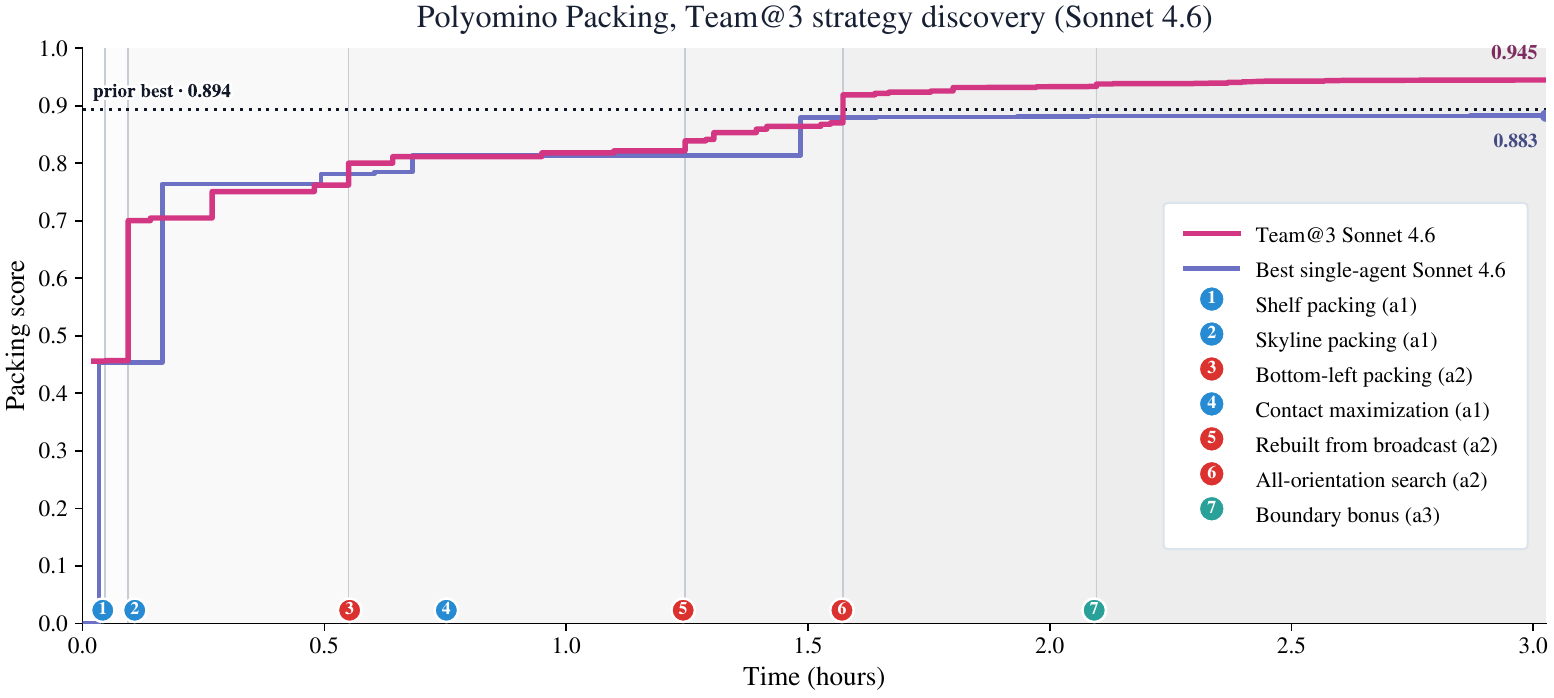}
\caption{\textbf{Communicating agents surpass the single-agent plateau and
set a state-of-the-art packing score.}
Best-so-far packing score for the highest-scoring run among 60 single-agent 
trials and 20 team@3 trials. Team@3 pulls decisively ahead at about 1.6 hours and 
reaches $0.945$, a new best score within the benchmark, above
$0.883$ for the single agent and the prior-best score of $0.894$.
Numbered markers identify method discoveries and their contributing agents.
Shaded bands track changes in the team's leading method.}
\label{fig:poly-timeline}
\end{figure}

We now turn to Frontier-CS~\citep{mang2026frontiercs} to test whether test-time communication
yields stronger performance on challenging algorithmic tasks. We study its polyomino packing
task, in which agents must pack a set of pieces into a rectangular box of minimal area.
The task is NP-hard, so agents must discover good heuristics and design a packing algorithm
rather than search exhaustively. The best known score of $0.894$ was obtained by
\citet{qu2026coral} using four Opus 4.6 agents.

We use Claude Sonnet 4.6 and Opus 4.6 agents under the same communication protocol as in
\texttt{ARC-AGI-3}.
We first evaluate Sonnet 4.6 under Frontier-CS's three-hour limit, with 60 single-agent
trials and 20 team@3 trials.
We then extend the time limit to 72 hours to study long-horizon performance and 
run 12 single-agent trials and 2 team@4 trials.
For each method, we report the trajectory of the run that achieves the highest score.
Our focus is on whether a method can discover a solution beyond the existing frontier,
where a single breakthrough matters even when other attempts are unsuccessful 
\citep{yuksekgonul2026learning}.

\begin{figure}[t!]
\centering
\includegraphics[width=\linewidth]{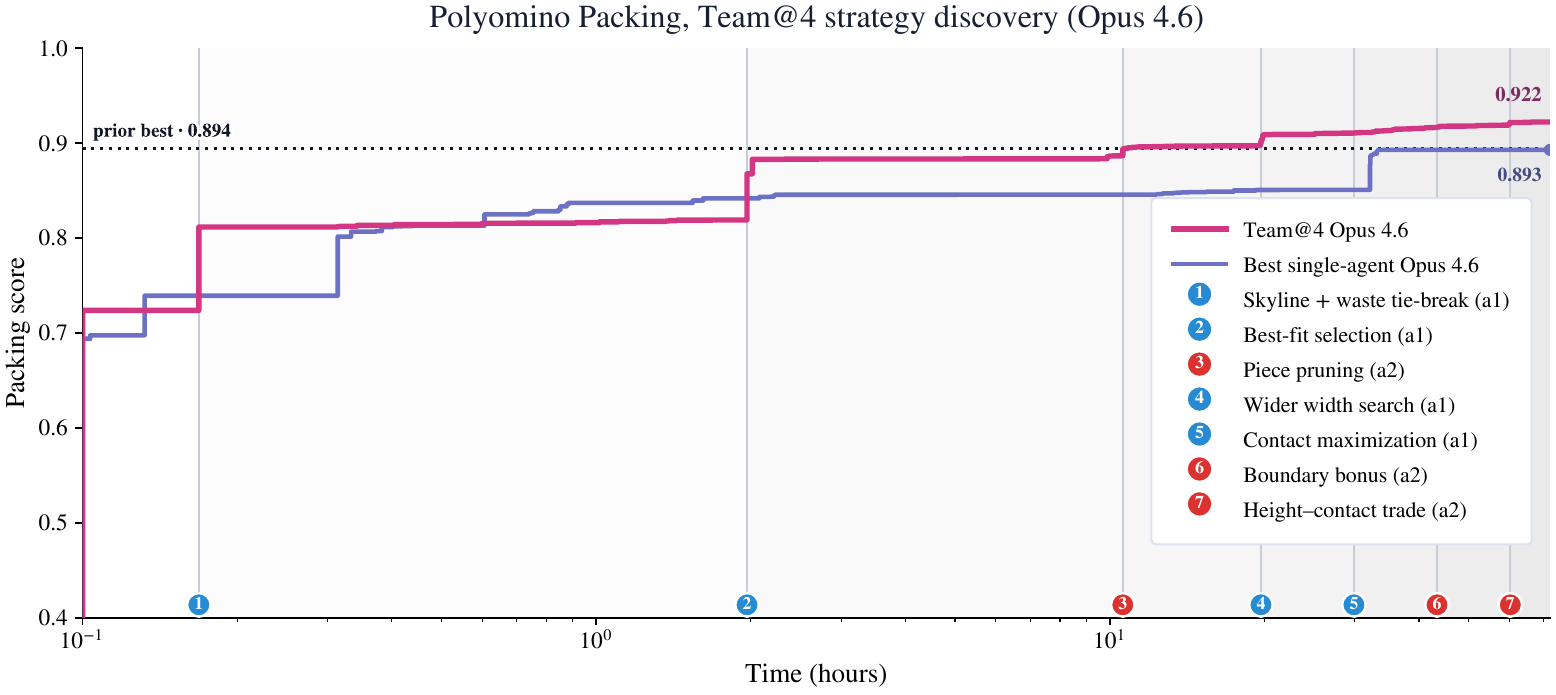}
\caption{\textbf{The Opus 4.6 team sustains its advantage over a longer horizon.}
Best-so-far packing score for the highest-scoring run among 12 single-agent trials and
two team@4 trials, each allowed up to 72 hours. Time is shown on a logarithmic axis.
The team takes a lasting lead at about two hours and reaches $0.922$, while the
single agent plateaus at $0.893$, just below the dotted prior-best reference of $0.894$.
Markers and shading follow Figure~\ref{fig:poly-timeline}.}
\label{fig:poly-opus}
\end{figure}

\begin{figure}[t!]
\centering
\includegraphics[width=\linewidth]{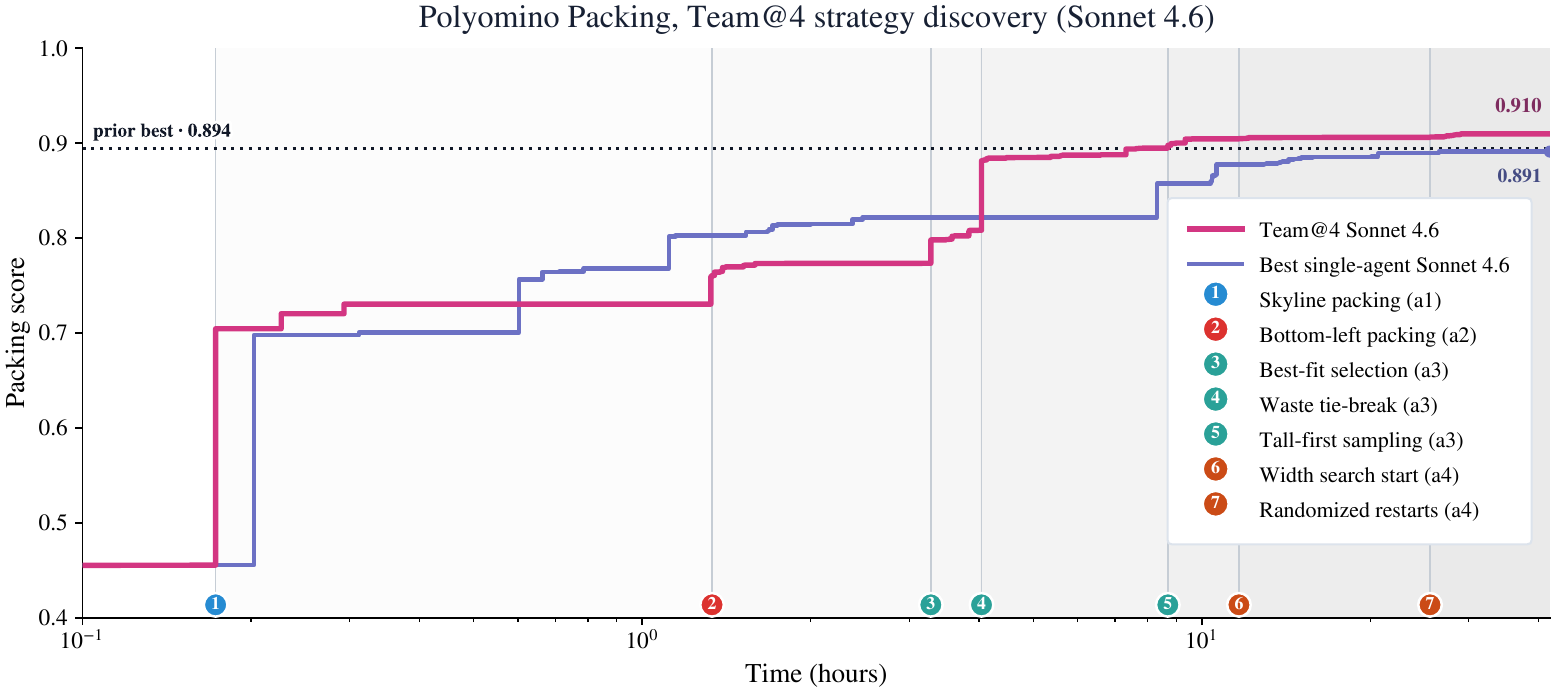}
\caption{\textbf{The longer Sonnet 4.6 team run also clears the single-agent plateau.}
Best-so-far packing score for the highest-scoring run among 12 single-agent trials and
two team@4 trials, each allowed up to 72 hours, with the selected team trajectory ending
at roughly 42 hours. Time is shown on a logarithmic axis.
The team takes a lasting lead at roughly four hours and reaches $0.910$, while the
single agent finishes at $0.891$. The dotted line marks the prior best of $0.894$.
Markers and shading follow Figure~\ref{fig:poly-timeline}.}
\label{fig:poly-sonnet}
\end{figure}

\paragraph{Official Frontier-CS results.}
Figure~\ref{fig:poly-timeline} shows that both runs reach roughly $0.82$ within the first hour, 
but the team breaks away at about 1.6 hours and continues improving as the single agent plateaus.
The final gap, $0.945$ versus $0.883$, more than halves the mean fraction of unused area, 
from $11.7\%$ to $5.5\%$, and carries a team of Sonnet 4.6 agents past the prior best of $0.894$.
The team’s advantage grows as the single agent’s progress slows, resembling the larger gap 
between team and independent agents at deeper levels of \texttt{ARC-AGI-3}.

\paragraph{Long-horizon results.}
Figures~\ref{fig:poly-opus} and~\ref{fig:poly-sonnet} show that the gap survives this
extension. Independent agents do improve further, but their best scores settle at
$0.893$ and $0.891$, still below the prior best. The teams take lasting leads at
roughly two and four hours, respectively, and continue to $0.922$ and $0.910$.
(The fact that team scores are lower than 0.945 is likely due to the small number of trials.)
Thus, the previous short timeframe does not fully explain the single-agent plateau.
These runs demonstrate that communication also outperforms independent agents even in
long-horizon settings, as long as communicating agents have enough time and compute to 
explore and share their discoveries.

\subsubsection{Qualitative Case Study on Communication}
\label{sec:polyomino-long}

In this subsection, we examine the successful three-hour Sonnet run that set a new state-of-the-art 
score of $0.945$ to qualitatively understand and assess how good communication leads to breakthroughs.
With the help of Opus 5, we analyze the run's traces and identify the key discoveries 
that led to the final score. We follow the numbered discoveries in
Figure~\ref{fig:poly-timeline}, using the agent labels shown in the figure.

\begin{figure}[t]
\centering
\includegraphics[width=0.86\linewidth]{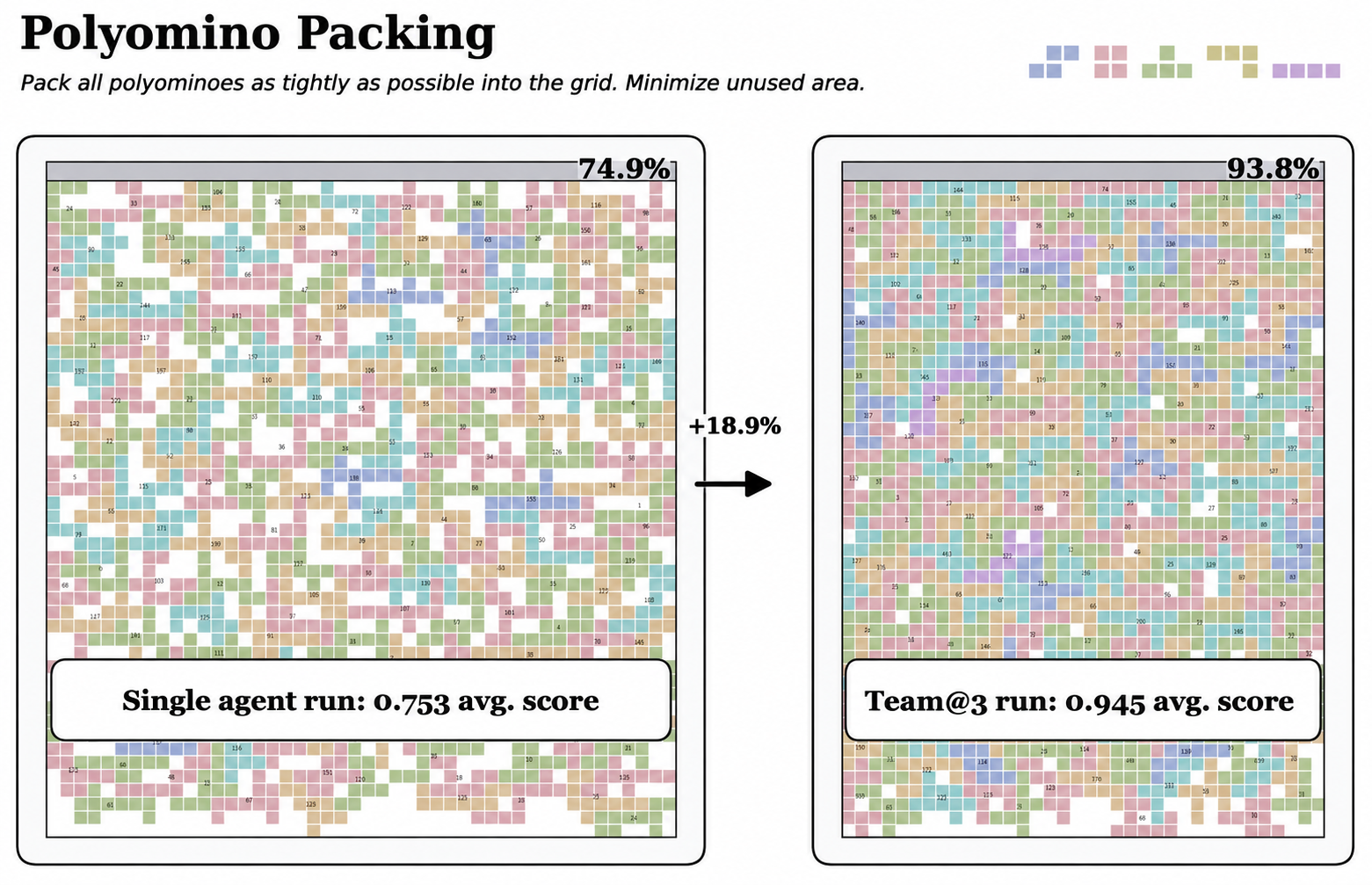}
\caption{\textbf{Communication yields a denser polyomino packing.}
The single-agent solution and the best Sonnet team@3
solution pack the same 172 pieces (1,654 cells) at $74.9\%$ and $93.8\%$ density,
respectively, a gain of 18.9 percentage points.
The labels $0.753$ and $0.945$ are the runs' benchmark-average scores.}
\label{fig:poly-before-after}
\end{figure}

\paragraph{From compact placement to filling gaps.}
Agent \textcolor[HTML]{268BD2}{\textbf{\texttt{a1}}} begins with \emph{shelf packing}, arranging pieces in rows, then
\emph{skyline packing}, placing them along the upper contour of the occupied region.
Agent \textcolor[HTML]{DC322F}{\textbf{\texttt{a2}}}'s \emph{bottom-left packing} searches for the lowest legal placement,
breaking ties toward the left. These methods successively lead the run at markers
1--3 in Figure~\ref{fig:poly-timeline}, but stall near $0.80$. 
Keeping placements low still leaves enclosed gaps that later pieces cannot fill.
Agent \textcolor[HTML]{268BD2}{\textbf{\texttt{a1}}} identifies this limitation and proposes \emph{contact maximization},
favoring placements that touch more already-occupied cells. Contact serves as a local
measure of fit, encouraging pieces to nest together rather than merely keeping the
current height small.

\paragraph{Building upon another agent's failed idea.}
The proposal for contact maximization appears at marker 4, but 
\textcolor[HTML]{268BD2}{\textbf{\texttt{a1}}}'s first implementation is slow and
the score barely changes for roughly half an hour. Agent \textcolor[HTML]{DC322F}{\textbf{\texttt{a2}}} 
then rebuilds the method from \textcolor[HTML]{268BD2}{\textbf{\texttt{a1}}}'s 
message in the communication logs without reading its code, describing the rule
as counting ``how many adjacent cells are already occupied for each piece that fits''
and choosing the highest-contact candidate. This independent implementation and continued
efficiency improvements raise
the score to nearly $0.87$ after marker 5.
Agent \textcolor[HTML]{DC322F}{\textbf{\texttt{a2}}} then evaluates every piece orientation within 
the placement lookahead, comparing how different orientations fit against
the existing packing. Running this \emph{all-orientation search} at full strength
produces the decisive jump past $0.90$ at marker 6.

The difficulty is making that search affordable within the evaluator's two-second
limit per instance. Examining more orientations improves placement quality but risks
a timeout, while conservative search budgets leave useful candidates unexplored.
The original traces show other teams considering the same search but struggling
with timeouts and keeping its budget too small. The strongest single agent also
explores contact maximization without reaching the same quality in three hours.
The successful transfer therefore includes substantial implementation work.
Agent \textcolor[HTML]{DC322F}{\textbf{\texttt{a2}}} takes a shared scoring idea and makes a more thorough search
practical under the execution limit.

\paragraph{Extending the shared objective.}
With the run above $0.90$, \textcolor[HTML]{2AA198}{\textbf{\texttt{a3}}} adds a \emph{boundary bonus} at marker 7.
The contact score now counts box walls as well as occupied cells, rewarding pieces
that fit snugly against edges and into corners. This extends the same objective
developed by \textcolor[HTML]{268BD2}{\textbf{\texttt{a1}}} and implemented by \textcolor[HTML]{DC322F}{\textbf{\texttt{a2}}}, and subsequent refinements
bring the run to $0.945$. The sequence makes the dependence between contributions
concrete. One agent identifies a better placement criterion, another turns it into
an effective search, and a third improves the criterion on top of that implementation.

Figure~\ref{fig:poly-before-after} illustrates the resulting improvement on one
172-piece instance. The team packs it at $93.8\%$ density, compared with $74.9\%$
for an illustrative single-agent solution, leaving far fewer gaps between pieces.
This single-agent run averages $0.753$ across the benchmark and is distinct from the
strongest run in Figure~\ref{fig:poly-timeline}.

\paragraph{Related exchanges over longer horizons.}
In the Opus run (Figure~\ref{fig:poly-opus}), \textcolor[HTML]{268BD2}{\textbf{\texttt{a1}}} introduces
\emph{best-fit selection}, comparing remaining pieces and candidate positions before
committing the best placement. Agent \textcolor[HTML]{DC322F}{\textbf{\texttt{a2}}} improves it and reduces its cost by
restricting the candidate pool to the largest piece sizes. Agent \textcolor[HTML]{268BD2}{\textbf{\texttt{a1}}} spends
the saved time searching more box widths, carrying the score past $0.90$.
A later contribution from \textcolor[HTML]{DC322F}{\textbf{\texttt{a2}}} changes how height and contact interact.
Instead of considering contact only after minimizing height, it scores their
weighted sum, allowing a slightly taller placement if it gains enough contact.
The other agents incorporate the change within an hour. Here, one agent's speed
improvement enables another's broader search, and the final refinement relaxes a
placement rule that earlier methods treated as fixed.

The Sonnet run (Figure~\ref{fig:poly-sonnet}) again develops through agents that did
not supply the initial solution. After \textcolor[HTML]{268BD2}{\textbf{\texttt{a1}}}'s skyline and \textcolor[HTML]{DC322F}{\textbf{\texttt{a2}}}'s
bottom-left methods, \textcolor[HTML]{2AA198}{\textbf{\texttt{a3}}} introduces best-fit selection to minimize the
increase in box height. It then breaks equal-height ties by wasted space,
producing the sharp jump near four hours, and adds tall-first sampling to pass $0.90$.
Agent \textcolor[HTML]{CB4B16}{\textbf{\texttt{a4}}} subsequently extends \textcolor[HTML]{DC322F}{\textbf{\texttt{a2}}}'s width heuristic.
When \textcolor[HTML]{268BD2}{\textbf{\texttt{a1}}} identifies unproductive parts of the deterministic width sweep,
\textcolor[HTML]{CB4B16}{\textbf{\texttt{a4}}} reallocates that time to randomized restarts.
As in the three-hour case, progress comes from successive changes to a shared
algorithm, including improvements to where it spends its limited search time.

\begin{takeaway}{Test-time communication breaks through single-agent performance plateaus}
For the polyomino packing task, single agents plateau near the prior best while four-agent
teams continue improving. The Opus team reaches $0.922$ rather than $0.893$,
and the Sonnet team reaches $0.910$ rather than $0.891$ within 72 hours.
The three-hour experiment shows that communicating teams find a better solution
than independent attempts with the same total agent-hour budget and establishes
a new state-of-the-art score of $0.945$.
After an initial coordination cost, progress proceeds as a relay in which different agents
introduce and refine the methods that move the shared solution forward.
\end{takeaway}

\subsection{MNIST Classifier Compression}
\label{sec:mnist}

\begin{figure}[h!]
    \centering
    \makebox[\linewidth][c]{%
    \includegraphics[width=1.0\linewidth]{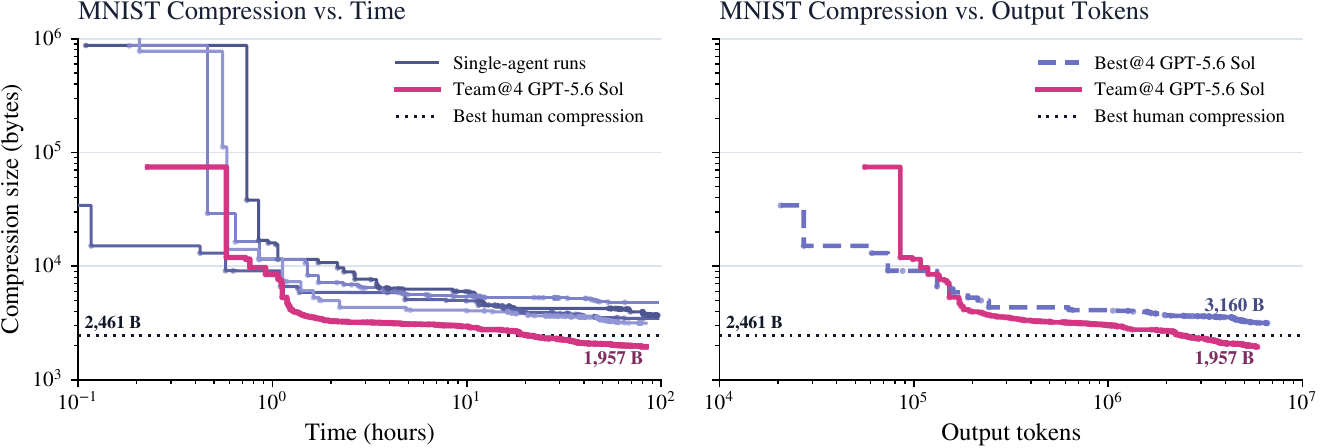}%
    }
\caption{\textbf{Communication leads to state-of-the-art ML compression.} 
Smallest submission with at least 99.4\% test accuracy over 96 hours, 
against wall-clock time and against output tokens, using GPT-5.6 Sol agents.
The advantage does not come from aggregating more attempts. 
All independent runs converge above the 2,461-byte best-known
human solution, while the team reaches 1,957 bytes. The team pays the coordination tax
here too, trailing for about the first hour and 100K output tokens.}
\label{fig:mnist-compression}
\end{figure}

We next move from abstract games and competitive programming style problems to a more realistic 
ML research engineering task. The main question here is whether communication benefits extend
to realistic multi-day engineering and how well communicating agents perform relative to the best-known human solutions.
In MNIST classifier compression, agents must produce the
smallest possible classifier that achieves at least $99.4\%$ test accuracy by minimizing the
compressed size of the complete submission that includes its inference code and model weights
in a 96-hour timeframe. The task combines empirical model development with 
deployment engineering, since agents must preserve accuracy
while compressing both the learned parameters and the code needed to use them.

Figure~\ref{fig:mnist-compression} plots the smallest qualifying submission found so far against
wall-clock time and total output tokens, with lower values indicating better compression. 
In the lower compute regime, the independent agents lead. 
They produce submissions below 15K bytes before
the team, which remains near 75K bytes until roughly the first half hour. Team@4 catches the
independent frontier after about one hour and 100K output tokens, then stays ahead as further
improvements steadily reduce its submission size. 
Communication therefore incurs an initial
communication cost, but converts later compute into more effective progress.

The difference becomes decisive over the full 96-hour run. The independent runs converge between
roughly 3KB and 5KB, with best@4 finishing at 3,160 bytes, and none crosses the 2,461-byte
best-known human result. 
Team@4 crosses that threshold after roughly 20 hours and two million output
tokens, then continues to 1,957 bytes, which is about $20\%$ smaller than the best-known human
submission. Because the advantage persists when progress is plotted
against total output tokens, it cannot be explained by the team merely generating more tokens.
Instead, test-time communication allows the four agents' work to be more effective by
diversifying their approaches, sharing ideas, and building on each other's progress. 
While it is certainly possible that a stronger compression could be found given more time or trials,
the result is a new state-of-the-art in MNIST classifier compression.

\subsubsection{Qualitative Study on Communication}
\label{sec:mnist-communication}

\paragraph{Finding a model worth sharing.}
The four agents begin with different bets. Agent \textcolor[HTML]{268BD2}{\textbf{\texttt{a1}}}
tries digit prototypes, \textcolor[HTML]{DC322F}{\textbf{\texttt{a2}}} fixed gradient features,
\textcolor[HTML]{2AA198}{\textbf{\texttt{a3}}} spectral models and then binary networks,
and \textcolor[HTML]{CB4B16}{\textbf{\texttt{a4}}} distilled CNNs.
After its prototype classifier misses the accuracy threshold,
\textcolor[HTML]{268BD2}{\textbf{\texttt{a1}}} changes the compression strategy.
A recurrent CNN reuses one convolutional cell nine times, buying depth without
storing a fresh set of filters at every step. Its 5,329-byte qualifying submission
persuades the other three agents to adopt this architecture within ten minutes.
Their separate searches now improve different parts of a shared model.

\paragraph{Giving lost features a way around the bottleneck.}
\textcolor[HTML]{2AA198}{\textbf{\texttt{a3}}} targets the final scoring layer, also called the classification head.
Its $10\times128$ matrix converts 128 features into ten digit scores using
1,280 learned parameters, plus ten bias parameters. It first compresses the 128 input features into eight learned
weighted sums before computing the ten digit scores. This stalls at
$99.30\%$ validation accuracy. To recover information lost in that compression,
it also sends selected original features directly to the final scoring layer.
Figure~\ref{fig:mnist-qualitative}(b) shows these two paths in red and teal. With five original features selected
using their covariance on the training data, the model reaches $99.42\%$
even before further fitting.
The resulting 3,120-byte model qualifies, and
\textcolor[HTML]{268BD2}{\textbf{\texttt{a1}}} explicitly adopts the design.
\textcolor[HTML]{CB4B16}{\textbf{\texttt{a4}}} subsequently combines seven learned weighted sums
with seven directly selected features at 2,790 bytes. Passing original features
to the scoring layer helps recover accuracy as the learned matrix shrinks.

\paragraph{Learning that fewer parameters can cost more bytes.}
Later, \textcolor[HTML]{DC322F}{\textbf{\texttt{a2}}} reduces the learned weighted sums to six on the
improved backbone. Additional directly selected features recover accuracy,
but the new parameters initially require more compressed bytes than those
of the larger head they replace.
\textcolor[HTML]{DC322F}{\textbf{\texttt{a2}}} finds that their values are less repetitive. It trains the projection
parameters using only the stored integers $-2,-1,0,1,2$, multiplied by a
learned scale factor for each weighted sum. With ten directly selected
features, this produces a qualifying 1,983-byte submission. Reducing the
number of parameters works together with restricting their stored values.

\paragraph{Rescuing a branch that has fallen behind.}
Meanwhile, \textcolor[HTML]{268BD2}{\textbf{\texttt{a1}}} makes the last two spatial
stages share their channel offsets. Its 1,993-byte model qualifies, but the
shared best has already reached 1,970 bytes.
Rather than discard the branch, it combines this backbone with
\textcolor[HTML]{DC322F}{\textbf{\texttt{a2}}}'s design with six learned weighted sums,
using routines originally written by \textcolor[HTML]{268BD2}{\textbf{\texttt{a1}}}
and extended by \textcolor[HTML]{DC322F}{\textbf{\texttt{a2}}}. Variants that pass ten or eleven original features directly to the scoring
layer fail the test threshold. Passing twelve succeeds while keeping the
projection parameters restricted to the same five stored integer values,
and \textcolor[HTML]{268BD2}{\textbf{\texttt{a1}}} advances the combined model to
1,960 and then 1,959 bytes. Figure~\ref{fig:mnist-qualitative}(a)
shows the adoption statements connecting these two branches.
\textcolor[HTML]{DC322F}{\textbf{\texttt{a2}}} then reorders the directly selected features and their matching scoring parameters to
reach 1,957 bytes without changing predictions. The final architecture emerges
from joining two branches, including one that could no longer win on its own.

\input{figures/mnist_qualitative}

\subsubsection{The 1,957-Byte MNIST Classifier}

\paragraph{Architecture.}
Figure~\ref{fig:mnist-qualitative}(b) shows a recurrent convolutional network
with 2,900 quantized parameters and 30 scale factors. The input is cyclically
shifted upward by one pixel. A $3\times3$ convolution maps the single input
channel to 32 channels, followed by SiLU to form $h_0$. Nine residual updates then reuse
the same depthwise $3\times3$ filters $D$ and pointwise $1\times1$ filters $C$.
For update $t$, the feature maps change according to
\begin{equation}
  \widetilde h_t=\operatorname{SiLU}\!\big(
    h_{t-1}
    +g_t\odot\operatorname{GN}_4\!\left(C\bigl(\operatorname{SiLU}(D(h_{t-1}))\bigr)\right)
    +b_t\big).
  \label{eq:mnist-residual}
\end{equation}
Here $D$ filters each channel separately, $C$ mixes channels, and
$\operatorname{GN}_4$ is GroupNorm with four groups, no learned affine
parameters, and $\epsilon=10^{-5}$. Each $g_t$ contains 32 learned gains,
broadcast over spatial positions. The offset $b_t$ is the 32-parameter vector
$A$ for updates 1--3 and the shared vector $B$ for updates 4--9.
The shortcut adds the unchanged input to the learned correction, making
these updates residual. All convolutions have stride one and no bias,
with padding one for $3\times3$ filters. After updates 3 and 6, $2\times2$
max-pooling with stride two forms $h_t$ from $\widetilde h_t$.
Otherwise $h_t=\widetilde h_t$. This gives three updates at each of
$28\times28$, $14\times14$, and $7\times7$ resolution.

After update 9, $4\times4$ average pooling with stride three produces
$32\times2\times2$ features, flattened in channel-first order into
$x\in\mathbb{R}^{128}$. The scoring layer combines six learned weighted
sums $Px$ with twelve directly selected features $x_S$,
\begin{equation}
  \mathrm{scores}=W\begin{bmatrix}Px\\x_S\end{bmatrix},
  \qquad P\in\mathbb{R}^{6\times128},
  \quad W\in\mathbb{R}^{10\times18}.
  \label{eq:mnist-head}
\end{equation}
The ordered, zero-based indices are
$S=(2,32,1,40,102,30,124,48,93,122,53,126)$.
These features are selected greedily using training feature statistics to recover
scoring information lost in the six-component approximation, and their indices
remain fixed for every image at inference.
Neither scoring matrix has a bias, and the largest score determines the digit.
The convolutional network has 1,952 parameters and the scoring matrices have
948. Filter reuse supplies depth without storing nine separate filter sets,
while the gains and offsets allow the updates to behave differently.

\paragraph{Training for compression.}
\textcolor[HTML]{DC322F}{\textbf{\texttt{a2}}}'s smaller scoring layer initially
costs more compressed bytes, motivating a change in how its weights are trained.
Quantization-aware training rounds weights during each forward pass while
learning the underlying weights and scales, restricting the 768 projection
parameters to scaled integers $-2,-1,0,1,2$. A small weight penalty
encourages small values alongside the classification and teacher-matching losses.
This makes the smaller scoring layer competitive by favoring repeated integer
values by exploiting redundancy in the learned weights. Across the final model,
$44.6\%$ of stored integers are $0$ or $\pm1$, measured before applying scales.
The compressed integer payload occupies 1,262 bytes, compared with 1,390 bytes
for fixed-width packing of the same values using each tensor's range.

The human reference, \texttt{tiny\_MNIST}~\citep{gandhi2024tinymnist}, uses
five separate $3\times3$ convolutions with widths $6,6,8,8,10$, BatchNorm,
ReLU, and a dense $90$-to-$10$ head. It has 3,130 parameters before
BatchNorm folding. Our reproduction folds BatchNorm into the convolutions
and applies per-channel 4-bit quantization, yielding 2,461 bytes at
$99.40\%$ accuracy. The agents' solution is 504 bytes
($20.5\%$) smaller and achieves $99.41\%$ test accuracy.

\begin{takeaway}{Test-time communication makes effective progress over a multi-day horizon and 
produces state-of-the-art ML compression results}
Team@4 establishes the best known MNIST classifier compression with a
1,957-byte submission that preserves at least $99.4\%$ test accuracy, about $20\%$
smaller than the best-known human solution. Despite trailing for roughly the first
hour and 100K output tokens, the communicating team overcomes its early coordination
cost over the 96-hour run and surpasses a threshold that no independent
run reaches, with best@$4$ finishing at 3,160 bytes.
\end{takeaway}

\subsection{Considerations for Effective Multi-Agent Communication}
\label{sec:effective-communication}

Prior work finds that tasks admitting decomposition into independent subtasks are
particularly amenable to multi-agent collaboration~\citep{kim2025scienceofscaling,gu2025agentgroupchatv2}.
Our tasks admit no such decomposition, yet communication helps regardless.
What makes it help is \emph{verified progress sharing}. Once a discovery is confirmed to be
useful, every agent can build on it rather than rediscover it independently, and no agent is
left exploring a direction the others have already exhausted.

In the following discussion, we provide a simple pedagogical model that formalizes this intuition.
We show how it can lead to an exponential separation between communicating and independent agents,
and also provide a small-scale experiment on Terminal-Bench 2.0~\citep{merrill2026terminalbench} 
where team@$k$ improves over single-agent attempts but not independent best@$k$,
illustrating the importance of the availability of an agent-accessible verifier.

\paragraph{From independent trajectories to cumulative progress.}
Consider a task requiring $m$ successive improvements, with an accessible verifier
reporting score $j/m$ after improvement $j$. Assume that continuation difficulty
depends only on the stage, improvements can be transferred freely, and agents conduct
fresh, independent searches after each transfer. Both methods can query the
verifier. Writing $X_{ij}$ for agent $i$'s search time at stage $j$, their completion
times would be
\[
T_{\mathrm{best}} = \min_{i\le k}\sum_{j=1}^{m}X_{ij}, \qquad
T_{\mathrm{team}} = \sum_{j=1}^{m}\min_{i\le k}X_{ij}.
\]
Thus $T_{\mathrm{team}}\le T_{\mathrm{best}}$. Independent sampling needs one agent
to complete the entire sequence quickly, whereas a team can use a different
agent's breakthrough at each stage. Communication changes a \emph{minimum of
sums} into a \emph{sum of minima}.
Figure~\ref{fig:verified-progress-sharing} illustrates this effect visually.

Suppose $X_{ij}\overset{\mathrm{iid}}{\sim}\operatorname{Exp}(\lambda)$,
which is the standard constant-rate, memoryless model used
in queueing theory \citep{shortle2018fundamentals}. Here we use this model as a pedagogical example, 
not an empirical claim about agents. The first of $k$ independent
discoveries arrives at rate $k\lambda$, giving
\[
T_{\mathrm{team}}\sim\operatorname{Erlang}(m,k\lambda),
\qquad
\mathbb{E}T_{\mathrm{team}}=\frac{m}{k\lambda}.
\]
The Erlang law (a Gamma distribution with integer shape) follows from summing
$m$ exponential waiting times. Doubling $k$ therefore halves the team's mean completion time.

Now give each agent runtime $\tau=\alpha m/\lambda$, where $1/k<\alpha<1$.
This is shorter than one agent's mean completion time but longer than the team's.
Both methods receive the same $k\tau$ agent-time budget.

\begin{figure}[t!]
\centering
\begin{tikzpicture}[
    x=0.68cm,y=1cm,
    >={Stealth[length=3mm]},
    agent/.style={line width=1.6pt, -{Stealth[length=3mm]}},
    divider/.style={gray!70, dashed, line width=0.9pt},
    boundary/.style={black, line width=1pt},
    lab/.style={font=\small},
    panellab/.style={font=\bfseries}
]

\definecolor{agentmagenta}{HTML}{D33682}
\definecolor{agentpurple}{HTML}{6C71C4}
\definecolor{agentteal}{HTML}{2AA198}

\begin{scope}[shift={(0,0)}]
    \node[panellab, anchor=west] at (-0.2,1.45) {(a) Independent agents};

    \draw[boundary] (0,-0.2) -- (0,1.0);
    \draw[boundary] (8,-0.2) -- (8,1.0);

    \node[lab, below] at (0,-0.2) {Start};
    \node[lab, below] at (8,-0.2) {Finish};

    \draw[agent, draw=agentmagenta] (0,0.75) -- (5.2,0.75);
    \draw[agent, draw=agentpurple]  (0,0.40) -- (8.0,0.40);
    \draw[agent, draw=agentteal]    (0,0.05) -- (6.6,0.05);

    \node[lab, text=agentmagenta, left] at (-0.1,0.75) {$a_1$};
    \node[lab, text=agentpurple,  left] at (-0.1,0.40) {$a_2$};
    \node[lab, text=agentteal,    left] at (-0.1,0.05) {$a_3$};
\end{scope}

\begin{scope}[shift={(10.5,0)}]
    \node[panellab, anchor=west] at (-0.2,1.45) {(b) Communicating agents};

    \draw[boundary] (0,-0.2) -- (0,1.0);
    \draw[boundary] (8,-0.2) -- (8,1.0);

    \draw[divider] (1.8,-0.2) -- (1.8,1.0);
    \draw[divider] (4.4,-0.2) -- (4.4,1.0);
    \draw[divider] (6.1,-0.2) -- (6.1,1.0);

    \node[lab, below] at (0,-0.2) {Start};
    \node[lab, below] at (8,-0.2) {Finish};

    \draw[agent, draw=agentmagenta] (0.0,0.75) -- (1.8,0.75);
    \draw[agent, draw=agentpurple]  (0.0,0.40) -- (1.1,0.40);
    \draw[agent, draw=agentteal]    (0.0,0.05) -- (0.8,0.05);

    \draw[agent, draw=agentmagenta] (1.8,0.75) -- (3.2,0.75);
    \draw[agent, draw=agentpurple]  (1.8,0.40) -- (4.4,0.40);
    \draw[agent, draw=agentteal]    (1.8,0.05) -- (3.5,0.05);

    \draw[agent, draw=agentmagenta] (4.4,0.75) -- (5.4,0.75);
    \draw[agent, draw=agentpurple]  (4.4,0.40) -- (5.6,0.40);
    \draw[agent, draw=agentteal]    (4.4,0.05) -- (6.1,0.05);

    \draw[agent, draw=agentmagenta] (6.1,0.75) -- (8.0,0.75);
    \draw[agent, draw=agentpurple]  (6.1,0.40) -- (7.0,0.40);
    \draw[agent, draw=agentteal]    (6.1,0.05) -- (7.4,0.05);

    \node[lab, text=agentmagenta, left] at (-0.1,0.75) {$a_1$};
    \node[lab, text=agentpurple,  left] at (-0.1,0.40) {$a_2$};
    \node[lab, text=agentteal,    left] at (-0.1,0.05) {$a_3$};
\end{scope}

\end{tikzpicture}
\caption{
\textbf{Verified progress sharing allows a team to advance faster.}
\textbf{(a)} Under independent parallel sampling, each agent advances only along its own trajectory, 
so success requires at least one agent to complete the entire trajectory independently.
\textbf{(b)} By sharing verified progress, the task is divided into successive phases, and at each checkpoint, 
all agents continue from the improved state, where a different agent may supply the next breakthrough.
}
\label{fig:verified-progress-sharing}
\end{figure}

\begin{proposition}[Exponential separation from progress sharing]
\label{prop:verified-progress}
Under this model, writing $I(a)=a-1-\log a$,
\[
\begin{aligned}
\Pr(T_{\mathrm{team}}\le\tau)
&\ge 1-e^{-mI(k\alpha)}, \;
\Pr(T_{\mathrm{best}}\le\tau)
&\le k e^{-mI(\alpha)}.
\end{aligned}
\]
For fixed $k,\alpha$, team@$k$ approaches one exponentially in $m$, while
best@$k$ success approaches zero exponentially.
\end{proposition}
The separation concerns completion probability at matched runtime, not an
exponential expected-time speedup. The proof is in
Appendix~\ref{app:verified-progress}.
The benefit also depends on preserving diverse continuations. 
Sometimes, communicating agents may converge on a subset of approaches,
which can reduce diversity and slow progress.
To model \emph{herding}, we can imagine the $k$ agents acting as essentially
$g\le k$ independent groups whose members duplicate the same search. 
The team's discovery rate becomes $g\lambda$
and its mean completion time becomes $m/(g\lambda)$. The high-success regime now
requires $g\alpha>1$. Sharing progress helps only to the extent that independent
search survives after sharing.

\paragraph{Ineffective communication in Terminal-Bench.}
Terminal-Bench 2.0 evaluates agents on hard, realistic tasks performed in
command-line environments~\citep{merrill2026terminalbench}. We evaluate Claude
Sonnet 4.6 through the GitHub Copilot CLI on the complete set of 89 tasks. 
Test-time communication consists of two independent team@2 trials, 
where both agents contribute to one final container state. 
The baseline consists of four independent single-agent trials.
This is a small-scale comparison with only a handful of trials,
so we treat it as descriptive evidence rather than a precise estimate of the
gap or its cause.

\begin{table}[t!]
\centering
\small
\caption{\textbf{Communication does not outperform pass@$k$ on Terminal-Bench 2.0.}}
\label{tab:terminalbench}
\begin{tabular}{@{}lcc@{}}
\toprule
Method & Mean Accuracy & Max Accuracy \\
\midrule
Single agent pass@1 & 52.53\% & -- \\
Independent pass@2  & \textbf{62.36\%} & \textbf{64.04\%} \\
Communicating team@2  & 60.67\% & 61.80\% \\
\bottomrule
\end{tabular}
\end{table}

As seen in Table~\ref{tab:terminalbench}, communication improves over an individual attempt, 
but does not outperform independent sampling (pass@$k$) with the same number of agents. 
This result can be partially explained by Proposition~\ref{prop:verified-progress}
when $m$ is effectively $1$. In Terminal-Bench, tasks require long sequences of actions, but
the official verifier runs after agent execution. Available feedback varies across tasks,
from package tests and direct performance measurements to checks of selected requirements.
For example, the public evaluator in \texttt{largest-eigenval} checks the eigenvector
equation and reports timings without verifying that the eigenvalue is dominant.
In \texttt{mailman}, the public evaluator checks subscription but omits announcement
delivery and unsubscription, which the official grader also tests.
Such checks can confirm progress on individual requirements without establishing
overall correctness.

In contrast, the packing and MNIST tasks expose numerical scores that agents can query repeatedly,
while ARC-AGI-3 supplies level-success feedback.
Thus, setting $m=1$ in Proposition~\ref{prop:verified-progress} gives us
$T_{\mathrm{team}}=\min_{i\le k}X_{i1}=T_{\mathrm{best}}$, and therefore shows no communication advantage,
even when the final outcome can be graded objectively. What matters is not evaluation alone,
but dense faithful verification at test-time.

A secondary disadvantage may be the loss of diversity. If communication causes $k$ agents
to behave like only $g<k$ independent groups, then $g$ end-to-end candidates
cover only $1-(1-p)^g$, below pass@$k$'s $1-(1-p)^k$, where $p$ denotes the 
success probability of an independent candidate. 
This penalty is especially relevant here because team@2 contributes one
shared final state, while pass@2 retains two isolated states and receives oracle selection
after evaluation. Premature convergence or interference in the shared state
can therefore make communication slightly worse than independent sampling.
Together, the results suggest that communication is most useful when feedback is accessible
and discriminative, though we do not exclude the possibility that test-time communication
could be useful when feedback is sparse, if the agents can construct their own intermediate 
verification signals or have better communication protocol harnesses.

\begin{takeaway}{Accessible verification helps agents build on shared progress}
Communication can help agents build upon successive breakthroughs when they can objectively
measure which partial solutions are worth building on, provided agents preserve diverse search
directions after sharing. Terminal-Bench 2.0 illustrates the challenges of sharing progress
with only partial access to verification. Available feedback varies
across tasks and may not reliably identify progress worth sharing, which may help explain
why team@2 improves over a single attempt but does not outperform pass@2.
\end{takeaway}

%% file: figures/mnist_qualitative.tex
\begin{figure}[t!]
\centering
\begin{minipage}[t]{0.475\linewidth}
\vspace{0pt}
{\small\textbf{(a) Communication log excerpts}\par}
\vspace{3pt}
\begin{tcolorbox}[colback=blue!3!white,colframe=blue!25!white,
  boxrule=0.5pt,arc=1.5mm,left=2mm,right=2mm,top=2mm,bottom=2mm]
\footnotesize
\textcolor[HTML]{268BD2}{\textbf{\texttt{a1}}}
\textbf{adopting \textcolor[HTML]{859900}{\texttt{a3}}'s smaller scoring layer}\par\smallskip
``slot 3 rank-8-plus-residual classifier factorization [\ldots]
measured 3,117 bytes, 99.50\% dev and 99.41\% sealed test clearly beats the dense-head frontier [\ldots]''
\end{tcolorbox}
{\footnotesize
\textcolor[HTML]{268BD2}{\textbf{\texttt{a1}}} adopts
\textcolor[HTML]{859900}{\textbf{\texttt{a3}}}'s design, combining learned weighted
sums with directly selected features.\par}
\vspace{5pt}
\begin{tcolorbox}[colback=blue!3!white,colframe=blue!25!white,
  boxrule=0.5pt,arc=1.5mm,left=2mm,right=2mm,top=2mm,bottom=2mm]
\footnotesize
\textcolor[HTML]{268BD2}{\textbf{\texttt{a1}}}
\textbf{borrowing \textcolor[HTML]{DC322F}{\texttt{a2}}'s scoring layer}\par\smallskip
``Adopted peer rank6 head after measured 1970-byte/99.42\% sealed frontier
[\ldots] 12 bypasses rather than peer 10.''
\end{tcolorbox}
{\footnotesize
Later, \textcolor[HTML]{268BD2}{\textbf{\texttt{a1}}}'s offset-sharing model
at 1,993 bytes trails the best-so-far of 1,970. Combining it with
\textcolor[HTML]{DC322F}{\textbf{\texttt{a2}}}'s six weighted sums and twelve
directly selected features reaches 1,959 bytes.\par}
\vspace{5pt}
\begin{tcolorbox}[colback=red!3!white,colframe=red!25!white,
  boxrule=0.5pt,arc=1.5mm,left=2mm,right=2mm,top=2mm,bottom=2mm]
\footnotesize
\textcolor[HTML]{DC322F}{\textbf{\texttt{a2}}}
\textbf{taking up the combined model}\par\smallskip
``The new 1,959-byte frontier combines tied stage biases with a
rank-6/twelve-bypass head and seals at 99.41\%. I'm adopting it [\ldots]''
\end{tcolorbox}
{\footnotesize
\textcolor[HTML]{DC322F}{\textbf{\texttt{a2}}} adopts the combined model and
reorders features and matching scoring parameters, reaching
\textbf{1,957 bytes}\par}
\end{minipage}\hfill
\begin{minipage}[t]{0.49\linewidth}
\vspace{0pt}
{\small\textbf{(b) Final classifier architecture}\par}
\vspace{6pt}
\centering
\begin{tikzpicture}[x=1cm,y=-1cm,
  font=\scriptsize,
  flow/.style={-{Stealth[length=1.6mm]},black!65,semithick},
  weight/.style={blue!55!black,densely dashed,thin},
  every node/.style={align=center,inner sep=1pt}]
\path[use as bounding box] (-.15,-.05) rectangle (6.4,10.55);
\fill[black!3] (2.65,0) rectangle (3.45,.8);
\draw[step=.2cm,black!12,very thin] (2.65,0) grid (3.45,.8);
\draw[black!75,line width=2.5pt,line cap=round,line join=round]
  (2.82,.17) -- (3.27,.17) -- (3.02,.64);
\node[anchor=east] at (2.5,.4) {Input\\$28\times28$};
\draw[flow] (3.05,.85) -- (3.05,1.15);
\foreach \k in {2,1,0} {
  \filldraw[fill=orange!15,draw=orange!65!black,thin]
    (2.78+\k*.08,1.2-\k*.06) rectangle (3.22+\k*.08,1.64-\k*.06);
}
\draw[step=.1467cm,orange!65!black,very thin] (2.78,1.2) grid (3.22,1.64);
\node[anchor=west] at (3.7,1.38) {Input filters $3\times3$\\$1\to32$ channels};
\filldraw[fill=blue!2,draw=blue!15,rounded corners=2mm]
  (-.05,2.35) rectangle (6.25,5.25);
\draw[flow] (3.05,1.72) -- (3.05,2.05) -- (.2,2.05) -- (.2,3.38) -- (.49,3.38);
\node[fill=white,text=blue!55!black] at (3.1,2.35) {\textbf{One filter set reused at all 9 steps}};
\foreach \cx/\sz/\res in {.95/.88/28,3.1/.67/14,5.25/.48/7} {
  \foreach \k in {2,1,0} {
    \filldraw[fill=blue!12,draw=blue!50!black,thin]
      (\cx-\sz/2+\k*.07,3.38-\sz/2-\k*.07)
      rectangle (\cx+\sz/2+\k*.07,3.38+\sz/2-\k*.07);
  }
  \draw[flow] (\cx-.35,2.96) .. controls (\cx-.85,2.48) and (\cx+.85,2.48)
    .. (\cx+.4,2.96);
  \node[fill=blue!2] at (\cx,2.68) {$\times3$};
  \node at (\cx,4.03) {$32\times\res\times\res$};
  \draw[weight] (\cx,4.18) -- (\cx,4.45);
}
\draw[flow] (1.58,3.34) -- (2.57,3.34);
\draw[flow] (3.6,3.34) -- (4.85,3.34);
\node at (2.08,3.66) {max-\\pool};
\node at (4.22,3.66) {max-\\pool};
\draw[weight] (.95,4.45) -- (5.25,4.45);
\node[text=blue!60!black] at (1.55,4.88)
  {Depthwise $3\times3$\\filters within channels};
\node[text=blue!60!black] at (4.65,4.88)
  {Pointwise $1\times1$\\mixes channels};
\draw[flow] (2.87,4.76) -- (3.32,4.76);
\draw[flow] (5.63,3.4) -- (6.48,3.4) -- (6.48,5.7) -- (3.1,5.7) -- (3.1,6.03);
\node[fill=white] at (4.87,5.7) {average pool\\$32\times2\times2$};
\foreach \j in {0,...,15} {
  \pgfmathtruncatemacro{\cellshade}{15+mod(\j*13,40)}
  \filldraw[fill=blue!\cellshade,draw=white,line width=.3pt]
    (1.02+\j*.26,6.07) rectangle (1.28+\j*.26,6.3);
}
\node[anchor=east] at (.9,6.18) {128\\features};
\draw[flow] (1.95,6.34) -- (1.38,6.88);
\foreach \row in {0,...,5} {
  \foreach \col in {0,...,7} {
    \pgfmathtruncatemacro{\cellshade}{12+mod(\row*17+\col*11,55)}
    \filldraw[fill=red!\cellshade,draw=white,line width=.2pt]
      (.62+\col*.185,6.94+\row*.11) rectangle (.805+\col*.185,7.05+\row*.11);
  }
}
\node[anchor=east,text=red!65!black] at (.5,7.27) {$P$};
\node[anchor=west] at (2.23,7.27) {$6\times128$};
\foreach \j in {0,...,11} {
  \draw[teal!45,line width=.35pt]
    (3.32+\j*.153,6.33) -- (3.24+\j*.225,8.16);
}
\node[fill=white] at (4.48,7.27) {Select 12};
\foreach \j in {0,...,5} {
  \draw[red!40,line width=.4pt]
    (.72+\j*.25,7.62) -- (.53+\j*.34,8.16);
  \filldraw[fill=red!22,draw=red!60!black,thin]
    (.53+\j*.34,8.23) circle (.075);
}
\foreach \j in {0,...,11} {
  \filldraw[fill=teal!18,draw=teal!65!black,thin]
    (3.24+\j*.225,8.23) circle (.065);
}
\node[text=red!65!black] at (1.38,8.53) {6 weighted sums $Px$};
\node[text=teal!65!black] at (4.48,8.53) {12 original features $x_S$};
\draw[flow] (1.38,8.7) -- (1.38,8.94) -- (2.2,8.94) -- (2.2,9.12);
\draw[flow] (4.48,8.7) -- (4.48,8.94) -- (4.05,8.94) -- (4.05,9.12);
\foreach \j in {0,...,17} {
  \ifnum\j<6\def\featurecolor{red}\else\def\featurecolor{teal}\fi
  \filldraw[fill=\featurecolor!25,draw=white,line width=.3pt]
    (1.02+\j*.23,9.16) rectangle (1.25+\j*.23,9.37);
}
\draw[flow] (3.1,9.4) -- (3.1,9.77);
\node[anchor=west] at (3.35,9.6) {$W:$ $18\to10$};
\foreach \j/\h in {0/.08,1/.12,2/.18,3/.11,4/.07,5/.13,6/.09,7/.44,8/.16,9/.1} {
  \ifnum\j=7\def\barcolor{orange!80}\else\def\barcolor{black!20}\fi
  \fill[\barcolor] (1.73+\j*.27,10.24) rectangle (1.9+\j*.27,10.24-\h);
  \node[font=\tiny,anchor=north] at (1.81+\j*.27,10.28) {\j};
}
\node[anchor=east] at (1.57,10.03) {Scores};
\node[anchor=west,text=orange!70!black] at (4.57,10.03) {argmax};
\end{tikzpicture}
\end{minipage}
\caption{\textbf{Combining complementary ideas produces the final MNIST classifier.}
\textbf{(a)} Two shared-log entries from \textcolor[HTML]{268BD2}{\textbf{\texttt{a1}}} and a subsequent excerpt
from \textcolor[HTML]{DC322F}{\textbf{\texttt{a2}}} show successive adoption of \textcolor[HTML]{859900}{\textbf{\texttt{a3}}}'s and \textcolor[HTML]{DC322F}{\textbf{\texttt{a2}}}'s scoring-layer
designs, followed by \textcolor[HTML]{DC322F}{\textbf{\texttt{a2}}} taking up the combined model.
\textbf{(b)} The same filter set,
comprising depthwise $3\times3$ and pointwise $1\times1$ convolution filters,
updates the 32 feature maps in three loops at each resolution (see Eq.~\ref{eq:mnist-residual}).
The scoring layer combines six learned weighted sums with twelve selected features.
The final network contains 2,900 quantized parameters and 30 float16
scale factors.}
\label{fig:mnist-qualitative}
\end{figure}

%% file: sections/related.tex
\section{Related Work}

\paragraph{Multi-agent systems.}
Multi-agent systems (MAS) are becoming increasingly popular and widely adopted.
As discussed in the introduction, some prior work~\citep{du2024multiagentdebate,chen2024reconcile} 
reports gains from multi-agent debate and belief exchange.
However, deliberation without a verifier or external feedback can amplify 
correlated errors or move the team away from its strongest member
\citep{smit2024goingmad,pappu2026experts}. 
In fact, independent sampling followed by aggregation remains a strong baseline
\citep{wang2023selfconsistency,li2024moreagents,choi2025debateorvote}. 
Benefits also depend on task structure and model strength.
\citet{kim2025scienceofscaling} find stronger gains on decomposable tasks, 
but these gains diminish as the underlying model becomes stronger with respect to the task.
While \citet{kim2025scienceofscaling} comprehensively cover a range of tasks and configurations,
our work studies the benefits of MAS in a more open-ended discovery setting with a clear verifier signal.
As discussed in Section~\ref{sec:effective-communication}, it is 
\emph{verified progress sharing} that allows us to avoid ``agent collapse''
that previous works have observed.
For more detailed surveys of MAS, we refer to \citet{ferrag2026llm, tran2025multi}.

\paragraph{Test-time discovery.}
Test-time discovery often uses LLMs (or agents) to generate, evaluate, and iteratively
improve candidate solutions, rather than producing an answer in a single attempt from a pretrained model.
Within this paradigm, \citet{novikov2025alphaevolve} keep the 
underlying LLMs frozen and accumulate progress
through evolutionary program search, evaluator feedback, and a database of generated attempts.
\citet{wang2026thetaevolve,yuksekgonul2026learning}, by contrast, 
train the model via reinforcement learning during discovery, 
updating the LLM using feedback obtained from the target problems. 
Recent work also focuses on custom agentic workflows, where agents are assigned
specialized roles to generate, critique, and refine hypotheses, often
orchestrated by a central agent~\citep{gottweis2026coscientist,yamada2025aiscientistv2}.
Our setting, however, is orthogonal to the works above as they focus on either
single-agent capabilities or multi-agent workflows without
establishing an advantage over single-agent systems under matched compute.

Recent work illustrates the potential of multi-agent communication for discovery.
\citet{anthropic2026cryptographicweaknesses} describe an improved attack on HAWK
whose key idea emerged through an exchange between two agents, while
\citet{openai2026navierstokes} report a Navier--Stokes result produced by a group
of roughly 10,000 concurrent agents.
Similarly, \citet{bianchi2026einsteinarena} explore mathematical discoveries by an open
community of heterogeneous agents sharing solutions and encouraging public discussion.
These are systems in which agents choose how to explore and
communicate without a fixed workflow.

These findings motivate controlled tests of whether communication improves
on independent search, especially under matched compute.
\citet{anthropic2026multiagentsystems} report more vulnerabilities from coordinating
agents than from independent agents, though Claude Mythos Preview uses comparable
tokens per vulnerability when both methods are evaluated on the same search scope.
On other tasks, they observe failures to integrate agents' work, conformity,
and failures to share decisive evidence.
Closely related to our setup is \citet{qu2026coral}, whose agents share discoveries
through persistent memory and outperform the best of four independent runs under
matched wall-clock budgets.
In contrast, we diagnose clear conditions for when test-time communication is advantageous, 
through compute-controlled studies of scaling coordinating agents on
\texttt{ARC-AGI-3} and long-horizon algorithm/ML tasks, including negative results
at low-compute budgets and on Terminal-Bench~2.0.

\paragraph{Harness optimization and orchestration.}
One line of work aims to learn the ideal topology and communication protocol of MAS.
Agent graphs and communication links can be optimized directly
\citep{zhuge2024gptswarm,zhang2025gdesigner}, while architecture search can
adapt the system to the task \citep{zhang2025maas,yun2026graphofagents}. A
second line controls when and which agents participate. Systems can select
teams or gate communication \citep{liu2024dylan,ghosh2026grade}, regulate
interaction to avoid harmful scaling \citep{wang2025mars,shao2026monoscale},
or add agents dynamically and decompose work into explicit subtasks
\citep{costa2026agentspawn,gu2025agentgroupchatv2}.
\citet{gu2025agentgroupchatv2} investigate agent diversity and observe weaker returns from homogeneous teams,
though heterogeneity may bring down the performance of the best model
in the team even when explicitly told who the best is~\citep{pappu2026experts}.

Hierarchical systems place these decisions with a central
coordinator. Coordinators can be trained to assign work to frozen agents
\citep{xu2026trinity,nielsen2026conductor}, while adaptive hierarchies can
delegate work and synthesize results, primarily for parallel decomposition
\citep{tang2026fugu,openai2026multiagent}. Recursive delegation has also been
used for long-context inference rather than multi-agent orchestration
\citep{zhang2025rlm}. Harness optimization is broader and need not involve
multiple agents. Executable or prompt-level agent loops can be optimized from
scores, traces, or revision feedback
\citep{lee2026metaharness,lee2026recursiveharness}. Prompt optimization has also
been studied for multi-agent protocols \citep{bai2026maspromptbench}.

Unlike these approaches, we do not study orchestration or multi-agent harness optimization. 
Instead, we use a fixed shared-workspace scaffold of identical, unassigned peers and 
compare its performance with compute-matched independent runs to isolate communication from
parallel exploration.

\paragraph{Agent coordination as interactive systems.}
We view sequential Monte Carlo as an interacting-population counterpart to
best-of-$k$, with partial solutions repeatedly scored and resampled
\citep{gordon1993novel,liu1998sequential,delmoral2004feynmankac}. Populations
degenerate over long horizons \citep{kong1994sequential,doucet2011tutorial},
while resampling concentrates shared ancestors, and harder problems require
larger populations \citep{liu1998sequential,snyder2008obstacles}. These
phenomena correspond to our gains at depth, herding failures, and non-monotone
returns to team size.

Related methods now guide language model inference.
Particle methods support probabilistic inference and test-time scaling
\citep{zhao2024twisted,puri2025rollout}, as well as constrained generation
\citep{loula2025syntactic}. Interacting particles have also been adapted to
diffusion models \citep{singhal2025fksteering,luo2026selfrewarding}, with
related path selection for diffusion language model decoding
\citep{lee2025lookahead}. The analogy is conceptual and suggests that
coordination is most useful over long horizons with informative intermediate
feedback and a diverse population.

%% file: sections/conclusion.tex
\section{Conclusion and Future Directions}

In this paper, we studied the capabilities of test-time communication, a multi-agent setting in which agents
pursue the same open-ended objective and exchange evidence through a shared
workspace. This allows agents to explore different approaches in parallel while
turning individual discoveries into cumulative progress. On \texttt{ARC-AGI-3}, team@3 matches
the solve rate of 13 independent agents and team@5 matches that of 33, with the
multiplier growing from $4.3\times$ to $6.6\times$ as the team scales. Communication also solves
a game that remains unsolved across single-agent trials and makes the average member of a
team as action-efficient as the best of independent agents. The same pattern emerges in
longer-horizon tasks. Team@$k$ reaches a packing score of $0.945$ on
Frontier-CS polyomino packing, above the prior best-known score of $0.894$, and produce a
1,957-byte MNIST classifier, improving on both the best@4 result and the 2,461-byte
best-known human solution. These gains emerge only after an initial coordination cost, but with adequate 
compute, communication can outperform independent agents in both performance and efficiency.

The agent-accessible verifier in each task is central to our experimental design. 
Our setting differs from multi-agent debate, which asks agents to reconcile answers without
necessarily grounding the exchange in environmental feedback.
Level completion, packing score, and compression size give agents an objective basis
for rejecting failures, comparing approaches, and building on improvements. 
Whether communication produces similar gains when feedback is sparse,
noisy, or subjective remains an important open question.

On the other hand, multi-agent harnesses may matter as much as single-agent harnesses. 
Our controlled design is deliberately narrow. 
We fix one communication harness and use homogeneous agents with identical instructions 
and no assigned roles or central orchestrator. 
This isolates communication, but it does not imply that roles, model diversity, communication topology, 
or orchestration are unimportant. Future work should consider how agents communicate, 
which peers receive their messages, or how teams should be organized.
These policies could be engineered or evolved, similarly to harness optimization
\citep{lee2026metaharness,lee2026recursiveharness}.
Our results provide a baseline for future work, and hopefully motivate further study of multi-agent collaboration 
and test-time communication.

%% file: sections/setup_appendix.tex
\section{Full Experimental Setup}
\label{app:setup}

This section provides the communication-prompt and implementation details deferred
from Section~\ref{sec:setup}, followed by task-specific evaluation details for the
experiments reported in the main text.
The tasks, prompts, communication protocol, and further details can be found in 
\url{https://github.com/jerryjonghopark/test-time-communication}.

\subsection{Communication prompts, runtime, and accounting}
\label{app:common-setup}

\paragraph{Communication-prompt summary.}
Beyond the shared-workspace instructions in Section~\ref{sec:method}, the prompt
specifies when agents may adopt a peer's approach. Adoption requires a measured improvement,
replication, an actionable alternative to a blocked approach, or final convergence.
After adoption, agents are instructed to preserve a meaningful variation and record
what they adopted and why.

For MNIST classifier compression, agents build and validate candidates in scratch space, then
recheck the current best under the lock before promoting a strictly better submission.
These are instructions to the agents, rather than automatic verification or adoption
by the harness.

\paragraph{Runtime versions.}
The ARC-AGI-3 and Terminal-Bench 2.0 experiments use GitHub Copilot CLI~1.0.54, the long-horizon packing
experiments use version~1.0.70, and MNIST classifier compression uses version~1.0.78.

\paragraph{Shared and per-agent resources.}
CPU and host-memory limits apply to the entire task container. All workers in a
team share those limits, rather than each receiving a separate container-sized
allocation. Independent trials each have their own container. The configured
allocations are shown below.

\begin{table}[t!]
\centering
\small
\caption{\textbf{CPU/GPU and container setup for each task.}}
\begin{tabular}{@{}llll@{}}
\toprule
Task & Solo container & Team container & Task GPU \\
\midrule
ARC-AGI-3 & 4 CPUs, 8 GiB & 4 CPUs, 8 GiB & None \\
Polyomino packing & 2 CPUs, 6 GiB & 2 CPUs, 6 GiB & None \\
MNIST compression & 3 CPUs, 8 GiB & 12 CPUs, 32 GiB ($k=4$) & A100 \\
Terminal-Bench 2.0 & Task-defined & Task-defined, shared & Task-defined \\
\bottomrule
\end{tabular}
\end{table}

The ARC and packing allocations do not scale with team size. MNIST scales the
container's CPU and host-memory limits by four for team@4, but does not enforce
separate per-worker host-memory partitions. Its CPU thread pools are set to three
threads per worker. 
Each agent is instructed to use at most 1,000 MiB of aggregate GPU memory across
all its processes. A 400-MiB PyTorch allocator setting leaves headroom for
CUDA context and library memory, but is not a hard limit on total GPU memory.
The per-agent GPU-memory rule is self-enforced. These task-compute resources are
distinct from model inference through the Copilot service and from ARC's separate
per-agent action budgets.

\subsection{Agent Communication Prompt}
\label{app:comm-prompt}

Every agent in a team receives the prompt below, in addition to the task
instruction. The text is stored as a template, and before a run starts the
harness substitutes the team size into \verb|{n}| and \verb|{nm1}| and the
shared-scratch paths into the remaining braced fields. Every agent in a trial
then receives the identical filled-in text, so the only asymmetry between
agents comes from the slot each one wins at runtime.

{\scriptsize
\begin{verbatim}
    {n} agents share this container and work the same task in parallel, all with this
    identical prompt. Search widely without herding, coordinate as you go, and keep
    improving until time runs out.

    Shared scratch (create on first use):
      - slots/approaches: {slots}/
      - findings:         {findings}
      - disconfirmations: {disconfirm}
      - score log:        {plateau}
      - coordination:     {coordination}   (empty -- conventions you author)
    Your private scratch is {base}/work-$S after you claim slot $S.

    1. CLAIM A SLOT AND PICK A DISTINCT APPROACH:
         mkdir -p {slots}
         for i in $(seq 0 {nm1}); do mkdir "{slots}/slot-$i" 2>/dev/null && S=$i && break; done
         mkdir -p {base}/work-$S
         touch {coordination}
         echo "<your approach + what you will deliberately not assume>" > "{slots}/slot-$S/approach"
         cat {slots}/slot-*/approach
       If a lower-numbered slot already took your approach, change yours. Cover a
       different part of the search space; do not agree early.

    2. ALWAYS BE ACTING. Never end a turn with only prose or a plan. Every turn must
       run a command that advances or tests the work: take a real action the task
       accepts and read back its result. Publishing, disconfirming, coordinating, and
       pivoting are bookkeeping around real actions, never a substitute for taking one.

    3. EVALUATE AND RECORD EVERY ATTEMPT. Score each change with the task's own
       scoring or feedback mechanism (find it in AGENT.md or the task instructions),
       then log the result:
         echo "[slot $S $(date -u +%H:%M:%S)] score=<score-or-progress> family=<approach>" >> {plateau}

    4. SHARE WITHOUT HERDING. Append concise findings with evidence and cost to
       {findings}, and label weak claims as weak. Spend part of your effort trying to
       FALSIFY the leading idea or your own, recording negative results in
       {disconfirm}. Do not write prose whose only purpose is to make peers copy you.

    5. COORDINATE ON SHARED RESOURCES. You share the graded output, common files, and
       the environment with peers who run blind to your session. Treat every shared
       resource as contested: before you touch one, re-check {coordination} and the
       resource's current state; after, confirm your change survived and did not just
       repeat a peer's. A collision is any overwritten, duplicated, or conflicting work
       that wastes effort. No coordination scheme is provided -- {coordination} is empty
       and yours to author: on a collision, write a convention there concrete enough for
       a peer to follow, that changes your next action, then follow it.

    6. HIGH BAR TO CONVERGE. Keep your own approach unless another clearly beats it on
       a measured, reproduced result, or yours is blocked, or the run is wrapping up.
       Even then, keep one real difference (a parameter, subcase, representation, or
       fallback) until the very end.

    7. NEVER STOP WHILE TIME REMAINS. A working result is not the finish line; the
       clock running out is the only acceptable reason to stop. Do not declare the
       task done, final, solved, or "at the ceiling" and go idle -- a suspected
       ceiling is a claim to disconfirm, not a reason to quit.

    8. BREAK PLATEAUS BY CHANGING FAMILY. You are plateaued when your best score has
       not strictly improved over 3 consecutive attempts. Then stop tuning
       and switch to a STRUCTURALLY DIFFERENT approach -- a different core principle or
       assumption, not a variant of the current one. Keep a short list of untried
       families in {base}/work-$S so you always have a next one ready. Read peers'
       approaches and {plateau} first and pick a family no active peer is on; adopting
       a peer who is also plateaued is not progress.

    Run a tight loop -- change -> evaluate -> record -> repeat -- without pausing. Do
    not use the internet, curl, wget, HTTP libraries, or secrets.
\end{verbatim}
}

\subsection{ARC-AGI-3}
\label{app:arc-setup}

\paragraph{Task and interface.}
Agents interact with the games by receiving integer-valued
grids and selecting from the available discrete controls or coordinate clicks.
Each level has its own cap, so unused budget
from a later level cannot rescue an agent stalled earlier.  We use the benchmark-native
cap of five times the corresponding human-action baseline, independently for every agent.

Every team member owns a separately authenticated ARC session. Sharing an action trace
does not copy state and does not clear a peer's level.
Time-dependent progress curves use the first logged completion of each level.

\paragraph{Synchronization.}
Team runs synchronize at intervals of half the current level's action budget,
rounded up to an integer number of actions. The harness tracks each agent's phase
by its level and the number of these intervals it has used. An agent ahead of the
slowest live teammate, either within a level or by reaching the next level, cannot
take another game action until its teammates catch up or leave the live set.
During this pause, agents can still inspect state, run shell commands, and exchange
notes. Refused game actions consume no budget. Finished, game-over, and
budget-exhausted agents do not block their teammates, and inactive sessions are
excluded after an idle timeout. This synchronization applies only to communicating
teams, not to best@$k$.

\paragraph{Closed-book condition.}
The task container requires network egress for the model control plane, but agents receive
no browser or web-search tool and are explicitly forbidden to fetch public solutions,
replays, or ARC pages.  Only the task prompt, local run files, the agent's own environment
observations, and teammates' within-trial notes are admissible.
We inspect run trajectories afterwards to verify that no agent accessed external ARC information.

\subsection{Frontier-CS polyomino packing}
\label{app:poly-setup}

\paragraph{Objective and scorer.}
Each of 70 hidden cases contains $n\in[100,10^4]$ edge-connected polyominoes of one to ten
cells. A program may reflect, rotate by multiples of $90^\circ$, and translate each piece.
All pieces must lie without overlap in one integer-grid rectangle. If $C_j$ is the total
number of occupied cells and $A_j=W_jH_j$ the returned area, the normalized case quality is
$C_j/A_j$ and the reported score is its mean over the fixed cases.  Invalid output causes
the submission to be rejected.  Candidate programs compile as GNU C++17 and receive 2 seconds and
256 MiB per case.

The hidden instance directory is mounted read-only inside a separate scorer service and
is never mounted into the agent container, which contains the task statement and a
submission client.  On each attempt, the client sends the candidate C++ source over an
internal endpoint. The scorer compiles and executes it, with case-output retention
disabled, and returns the aggregate score together with compact status, timing, and
scoring metadata. The scorer never returns the instances or program outputs. Agents may submit repeatedly
without penalty, but cannot inspect the evaluation data.  Every attempt, timestamp,
status, and score is recorded in a submission ledger. The final verifier
retains the highest valid scored submission in the run, protecting an earlier champion
from a broken final edit.

\subsection{MNIST classifier compression}
\label{app:mnist-setup}

\paragraph{Data and artifact contract.}
The official 60,000-image MNIST training split is partitioned once into 55,000 training
and 5,000 development images. The agent image contains only these two partitions, while the
official 10,000-image test archive remains on the host and is never copied into the task
container.  Agents may train only on the 55,000 images, use the development set for
selection, and query the sealed oracle only after a full-development accuracy of at least
$99.4\%$.

An oracle request snapshots the canonical submission and passes that snapshot to a
host-owned evaluator.  Each request runs in a fresh container with networking disabled, a
read-only root filesystem, privilege escalation disabled, and the submission mounted
read-only.  Within it, a trusted verifier privately shuffles the test set, makes the test
archive and labels unreadable to submitted code, and invokes inference under an
unprivileged user.  The submission receives only batches of test images and writes
predictions to isolated scratch space. The trusted verifier alone reads the labels and
computes accuracy.  The container is discarded after evaluation, and the oracle returns
only aggregate accuracy. It never returns examples, labels, predictions, or per-example errors.
Final grading uses the same isolation boundary, with a 600-second limit for classifying
the complete test set.

The graded directory must include a \texttt{predict.py} entry point and every weight,
table, constant, generator, and decoder it needs.  We normalize file order and metadata,
form a deterministic tar archive, and apply gzip level~9. Training code outside
the submission and the preinstalled numerical runtime are not charged.

\paragraph{Open-source reference.}
The 2,461-byte reference in Section~\ref{sec:setup} is our benchmark-format reproduction
of the open-source \texttt{tiny\_MNIST} model~\citep{gandhi2024tinymnist}.  The published
architecture has 3,130 trainable parameters and reports accuracy above $99.4\%$.  We
retrained it from the published recipe, folded BatchNorm exactly into the adjacent
convolutions, and quantized the resulting weights per output channel to four bits.

For the deployable artifact, we bit-pack this state, use a compact decoder, and for fair comparison,
code-golf the inference path by eliminating general training machinery, redundant metadata,
whitespace, and intermediate file structure while preserving predictions. The complete normalized archive,
including executable inference code, is 2,461 bytes under the same deterministic gzip-9
metric used for every submission.  The reproduced artifact attains exactly $99.40\%$ on
the sealed test set.  Thus 2,461 bytes is a measured end-to-end submission size, not the
size of the upstream training checkpoint.

\paragraph{Environment.}
The pinned stack is Python~3.12.13, PyTorch~2.5.1 with CUDA~12.4,
torchvision~0.20.1, NumPy~2.5.1, SciPy~1.18.0, and scikit-learn~1.9.0.
Agents train from scratch on the provided data. Downloading additional data or
pretrained weights is prohibited, and inference evaluation has no network access.

\paragraph{Checkpoints and trajectory accounting.}
Agents are instructed to save each smaller submission that passes the full
development-set check, together with its result, timestamp, provenance, and
recomputed archive size. Saved checkpoints are audited, and the paper figures use
only submissions confirmed at $\geq99.4\%$ on the sealed test set.
Figure~\ref{fig:mnist-compression} compares four independent GPT-5.6~Sol runs with
one team@4 run. In the elapsed-time panel, best@4 at time $t$ is the smallest
qualifying submission found by any of the four independent agents within its own
first $t$ hours. The output-token panel uses the same best-so-far trajectory,
placing each improvement at the sum of tokens consumed by all four independent
runs by that elapsed time. Team tokens likewise sum all four workers. Neither
curve counts only the tokens of the agent that produced an improvement.

\subsection{Terminal-Bench 2.0}
\label{app:tb-setup}

\paragraph{Task environments and execution.}
The harness loads \texttt{terminal-bench@2.0} through Harbor and runs each task in
its Docker environment. CPU, host-memory, storage, and GPU requests come from the
individual task configuration, rather than a uniform benchmark-wide allocation.
The harness does not automatically multiply these requests by the number of
workers. In a team@2 trial, both workers share the task container and its resource
limits. Agent execution and verifier timeouts are configured separately by each
task, so there is no single wall-clock horizon analogous to the packing or MNIST
experiments. The verifier runs after agent execution and produces the trial's
reward from the resulting environment.

\paragraph{Aggregation.}
For Table~\ref{tab:terminalbench}, pass@1 averages the four independent outcomes
for each task. The saved independent results are grouped into two batches of two
attempts. Pass@2 counts a task as successful within a batch if either attempt
passes, then averages the two batch accuracies. Team@2 averages the two team-trial
outcomes for each task. Unfinished tasks are assigned zero, and all 89 tasks
remain in the denominator. Max Accuracy selects the higher batch or team-trial
accuracy over the whole benchmark, rather than selecting a different batch or
team trial for each task.

%% file: sections/proof_appendix.tex
\section{Analysis for Verified Progress Sharing}
\label{app:verified-progress}

This appendix gives the calculations behind the conceptual model and
Proposition~\ref{prop:verified-progress}. The mathematics is largely standard rather
than new. It combines elementary facts about exponential order statistics,
Erlang sums, and Chernoff bounds. We keep the model deliberately simple so that
it isolates how reusable, verified progress can change the value of parallel
search. It is not intended as an empirical model of agent completion times or
as a general theory of multi-agent communication.

\subsection{Idealized progress sharing}

We first compare independent and communicating agents under the same collection
of stage-level search times. Recall that $X_{ij}$ is the time agent $i$ would
take to discover the next improvement at stage $j$. A communicating team uses
the first discovery at every stage. Therefore, for every agent $i$,
\[
T_{\mathrm{team}}
=
\sum_{j=1}^{m}\min_{\ell\le k}X_{\ell j}
\le
\sum_{j=1}^{m}X_{ij}.
\]
Taking the minimum over agents gives
$T_{\mathrm{team}}\le T_{\mathrm{best}}$. This is the pointwise comparison in
the main text. It captures the distinction between combining stage-level
breakthroughs and selecting one complete trajectory after all runs finish.

Under the modeling assumption
$X_{ij}\overset{\mathrm{iid}}{\sim}\operatorname{Exp}(\lambda)$, define the
waiting time for the team's next verified improvement as
$M_j=\min_{i\le k}X_{ij}$. Its survival probability is
\[
\Pr(M_j>t)
=
\prod_{i=1}^{k}\Pr(X_{ij}>t)
=
e^{-k\lambda t},
\qquad t\ge0.
\]
Thus each $M_j$ is distributed as $\operatorname{Exp}(k\lambda)$. Independence
across stages then gives
\[
T_{\mathrm{team}}
\sim\operatorname{Erlang}(m,k\lambda),
\qquad
\mathbb{E}T_{\mathrm{team}}=\frac{m}{k\lambda}.
\]
For comparison, each independent agent must complete all $m$ stages itself. Its
completion time $T_i=\sum_{j=1}^{m}X_{ij}$ satisfies
\[
T_i\sim\operatorname{Erlang}(m,\lambda),
\qquad
T_{\mathrm{best}}=\min_{i\le k}T_i.
\]
These distributions provide the ingredients for the fixed-budget comparison.

\subsection{Completion probability at a fixed budget}

We now prove Proposition~\ref{prop:verified-progress}. The only additional
ingredient is a standard exponential tail bound for an Erlang random variable.

\begin{proof}[Proof of Proposition~\ref{prop:verified-progress}]
Let $S\sim\operatorname{Erlang}(m,1)$. Its moment-generating function is
\[
\mathbb{E}[e^{\theta S}]
=
(1-\theta)^{-m},
\qquad \theta<1.
\]
For $a>1$, apply Markov's inequality with
$\theta=1-1/a>0$ to obtain
\[
\Pr(S\ge am)
\le
e^{-\theta am}(1-\theta)^{-m}
=
e^{-mI(a)},
\]
where $I(a)=a-1-\log a$. For $0<a<1$, the same choice gives
$\theta<0$. Since $S\le am$ implies
$e^{\theta S}\ge e^{\theta am}$, Markov's inequality gives the corresponding
lower-tail bound. Hence
\begin{equation}
\begin{aligned}
\Pr(S\ge am)&\le e^{-mI(a)}
&&\text{for }a>1,\\
\Pr(S\le am)&\le e^{-mI(a)}
&&\text{for }0<a<1.
\end{aligned}
\label{eq:app-erlang-tails}
\end{equation}

Set $\tau=\alpha m/\lambda$, with $1/k<\alpha<1$. Since
$k\lambda T_{\mathrm{team}}$ has the same distribution as $S$ and
$k\alpha>1$,
\[
\Pr(T_{\mathrm{team}}>\tau)
\le e^{-mI(k\alpha)}.
\]
Likewise, $\lambda T_i$ has the same distribution as $S$, so $\alpha<1$
gives
\[
\Pr(T_i\le\tau)\le e^{-mI(\alpha)}.
\]
Taking a union bound over the $k$ independent agents yields
\[
\Pr(T_{\mathrm{best}}\le\tau)
=
\Pr\!\left(\bigcup_{i=1}^{k}\{T_i\le\tau\}\right)
\le k e^{-mI(\alpha)}.
\]
Combining the two bounds proves
\[
\Pr(T_{\mathrm{team}}\le\tau)
\ge1-e^{-mI(k\alpha)},
\qquad
\Pr(T_{\mathrm{best}}\le\tau)
\le k e^{-mI(\alpha)}.
\]
Because $I(a)>0$ for $a>0$ and $a\ne1$, both exponents are strictly
positive for fixed $k$ and $\alpha\in(1/k,1)$.
\end{proof}